\def\CameraReadyClean{}
\documentclass{article}

\usepackage[main, final]{neurips_2026}

\usepackage[utf8]{inputenc}
\usepackage[T1]{fontenc}
\usepackage{hyperref}
\usepackage{url}
\usepackage{booktabs}
\usepackage{amsfonts}
\usepackage{amsmath}
\usepackage{amssymb}
\usepackage{nicefrac}
\usepackage{microtype}
\usepackage{xcolor}
\ifdefined\CameraReadyClean
  \colorlet{red}{black}
  \colorlet{orange}{black}
  \colorlet{blue}{black}
  \colorlet{gray}{black}
  \hypersetup{hidelinks}
\fi
\usepackage{graphicx}
\usepackage{subcaption}
\usepackage{wrapfig}
\usepackage{algorithm}
\usepackage{algorithmic}
\usepackage{amsthm}
\usepackage{natbib}

\newtheorem{theorem}{Theorem}
\newtheorem{proposition}[theorem]{Proposition}

\newcommand{\E}{\mathbb{E}}

\newcommand{\D}{\mathcal{D}}

\title{Feature Information Dynamics in Diffusion}

\author{
  \textcolor{red}{Jia-Shu Pan$^{1}$, Tao Zhang$^{1,2}$, Yufei Huang$^{1,2}$, Yanjun Sheng$^{3}$, and Tailin Wu$^{1}$} \\
  \textcolor{red}{$^{1}$Department of Artificial Intelligence, School of Engineering, Westlake University} \\
  \textcolor{red}{$^{2}$Zhejiang University} \\
  \textcolor{red}{$^{3}$Research School of Astronomy and Astrophysics, Australian National University} \\
  \textcolor{red}{\texttt{panjiashu@westlake.edu.cn, wutailin@westlake.edu.cn}}
}

\begin{document}

\maketitle

\begin{abstract}
Diffusion models generate data through a continuum of denoising problems, and are widely observed to reveal coarse structure before fine detail.
Yet, this intuition is mostly empirical and qualitative.
We introduce \emph{feature information dynamics}, an information-theoretic framework for localizing when a feature is generated during diffusion.
Using the I-MMSE identity, we connect the rate of feature mutual information change to a gap between optimal unconditional and feature-conditional denoising losses, yielding practical estimators for feature information density.
We further develop a chained decomposition that separates shared from incremental information in a feature hierarchy. We use this framework first to quantitatively confirm spectral autoregression in pixel diffusion, and then to extend the analysis beyond frequency: under a class$\,\to\,$mask$\,\to\,$Canny conditioning chain, the per-feature information densities differ across pixel, SDVAE, VAVAE, and RAE, exposing fundamental differences between these representations and suggesting that ordered generation could be beneficial for training diffusion models. \textcolor{red}{Our code is available at \url{https://github.com/AI4Science-WestlakeU/feature-information-dynamics}.}
\end{abstract}


\section{Introduction}
\label{sec:introduction}

Diffusion models have achieved remarkable success in visual generation by decomposing the synthesis of complex distributions into a sequence of denoising tasks across varying signal-to-noise ratios~\citep{ddpm, cfm}.
A growing body of empirical evidence suggests that these models act as hierarchical feature generators, progressively building structure across the denoising trajectory~\citep{ddpm,time_rep,emergence,sfd,latent_forcing}. This temporal organization has broad implications for representation learning and generative modeling, rather than being merely descriptive.\footnote{\textcolor{red}{Possible uses for this diagnostic are discussed in Section~\ref{sec:discussion:possible-uses}.}} For example, \citet{time_rep} utilizes diffusion models to learn disentangled representations, and recent generative methods achieve better generation performance by prioritizing global semantic structure over local appearance in diffusion~\citep{sfd,latent_forcing}. Yet, what remains missing is a general, quantitative account of \emph{feature information dynamics}: when, and in what order different features are generated along the diffusion trajectory. We address this question by localizing a structured hierarchy of features across different noise levels in the diffusion process.

We state our contributions upfront:

\begin{enumerate}
  \item \textbf{Feature information dynamics via I-MMSE.}
    \textcolor{red}{Taking classical I-MMSE as a starting point, we define the SNR-resolved information between a specified feature and the noisy image, and relate its derivative to the gap between optimal unconditional and feature-conditional denoising losses.}
    We derive a family of equivalent estimators for this information density, and further exploit the analytic properties of the MMSE curve to give a practical procedure for obtaining the information profile of a general feature.
    We compare these estimators in detail (Appendix~\ref{app:estimators}).

  \item \textbf{Chained information decomposition.}
    When a hierarchy of features is present, the shared information between features must be considered to reveal the full information structure.
    Building on the above theory, we decompose feature information in a chained manner across the diffusion timeline.
    We experimentally verify one widely invoked intuition and report one surprising finding:
    \emph{(i)}~pixel-space image diffusion models indeed approximate spectral autoregression;
    \emph{(ii)}~in pixel-space diffusion, structure can enter earlier and more strongly than class identity, challenging a simple semantic-first account of the denoising trajectory.

  \item \textbf{Representations shape feature information dynamics.}
    Going beyond frequency, we apply the chained framework to a general visual feature hierarchy on ImageNet-256, holding the conditioning chain and evaluation protocol fixed across pixel-space, SDVAE, VAVAE, and RAE diffusion models. The information dynamics differ significantly across representations: in pixel space, mask information enters earlier and more strongly than class identity; only RAE produces a clean class$\,\to\,$mask$\,\to\,$Canny ordering. These findings recast representation choice as a controllable knob over feature information dynamics rather than a black box on top of FID.
\end{enumerate}


\section{Related Work}
\label{sec:related-work}

\paragraph{Information theory and diffusion models.}
The relationship between information theory and deep learning has been mutually enriching: information-theoretic principles have shaped learning objectives and representation learning~\citep{ib,cpc}, while neural networks have in turn enabled scalable estimation of previously intractable quantities such as mutual information~\citep{mine}.
Diffusion models sit naturally at this intersection: their forward process defines a Gaussian channel, making them particularly amenable to an information-theoretic analysis.
\textcolor{red}{\citet{info_diffusion} generalized the well-known information-theoretic I-MMSE identity to provide an interpretation of the diffusion loss.}
\citet{minde}
and \citet{mmg}
leverage the score difference to estimate \emph{static} mutual information quantities, but do not study the information \emph{dynamics} across the diffusion timeline.
Most closely related to our work, \citet{cls_spec} 
use conditional entropy to study how class identity is progressively specified during diffusion sampling.
Their analysis, however, is restricted to categorical features, and tied to the SDE formulation of diffusion models.
Our framework instead provides a general information dynamics framework for arbitrary features, along with practical estimators.

\paragraph{Feature information dynamics in diffusion.}
The question of when different visual information appears in diffusion has been present since the earliest DDPM observations: intermediate $\hat{x}_0$ predictions already suggested that large-scale structure emerges before fine detail~\citep{ddpm}.
In pixel space, this intuition receives a concrete spectral explanation: natural images approximately follow power-law spectra~\citep{img_power}, so Gaussian corruption separates low-frequency global structure from high-frequency texture, motivating the view of diffusion as \emph{spectral autoregression (AR)}~\citep{spec_ar}. Subsequent work provides geometric and mechanistic analyses of timestep-dependent structure and concept emergence~\citep{latent_riemann,emergence}. However, these explanations remain qualitative and cannot be readily generalized to diffusion models operating in representations beyond pixels. This matters because latent representations are now the mainstream compromise between computational efficiency and generation quality~\citep{ldm}, and the choice of representation can substantially affect training efficiency~\citep{vavae, rae}.
A quantitative, representation-agnostic account of \emph{feature information dynamics} is still missing.

\section{Feature Information Density}

Suppose we have a dataset of continuous random vectors $\boldsymbol{X}\sim p(\boldsymbol{x})$, e.g. images, and a corresponding discrete or continuous feature $Y\sim p(y\mid \boldsymbol{x})$, e.g. class labels or frequency components.
Our goal is to quantify, along an additive Gaussian corruption process of $\boldsymbol{X}$, \emph{how the information carried by $Y$
is distributed at different noise levels}.
To this end, we first introduce a classical Gaussian channel and minimum mean square error (MMSE),
then extend the I-MMSE identity to define \emph{feature information density} as the gap between the unconditional and feature-conditional MMSE.
Finally, we instantiate this information-theoretic framework with diffusion models and present estimators that also apply beyond diffusion models.

\subsection{Gaussian channel and MMSE}

Let $\boldsymbol{X} \in \mathbb{R}^d$, $\boldsymbol{X} \sim p(\boldsymbol{x})$, denote the clean data, and let $Y\sim p(y\mid \boldsymbol{x})$ denote a feature of interest.
We consider the Gaussian channel
\begin{equation}
    \boldsymbol{X}_\gamma = \sqrt{\gamma}\, \boldsymbol{X} + \boldsymbol{N},
    \qquad
    \boldsymbol{N} \sim \mathcal{N}(\boldsymbol{0}, I),
    \quad \boldsymbol{N} \perp\!\!\!\perp (\boldsymbol{X}, Y),
\end{equation}
where $\gamma \in \mathbb{R}_{\ge 0}$ is the signal-to-noise ratio (SNR).

Given any (possibly empty) side information $S$, we aim to denoise $\boldsymbol{X}$ from the noisy observation $\boldsymbol{X}_\gamma$ and $S$ under squared-error loss.
The minimum achievable error over all denoisers $\hat{\boldsymbol{x}}(\boldsymbol{X}_\gamma, \gamma, S)$ is the \emph{minimum mean-squared error} (MMSE),
\begin{equation}
    \operatorname{mmse}_{\boldsymbol{X} \mid \boldsymbol{X}_\gamma, S}(\gamma)
    :=
    \min_{\hat{\boldsymbol{x}}}\;
    \mathbb{E}_{\boldsymbol{X}, \boldsymbol{X}_\gamma, S}\!\left[
        \left\| \boldsymbol{X} - \hat{\boldsymbol{x}}(\boldsymbol{X}_\gamma, \gamma, S) \right\|_2^2
    \right].
\end{equation}
A classical result is that this minimum is attained by the conditional expectation $\hat{\boldsymbol{x}}^\star(\boldsymbol{X}_\gamma, \gamma, S) = \mathbb{E}[\boldsymbol{X} \mid \boldsymbol{X}_\gamma, S]$, which we call the \emph{optimal denoiser}; equivalently,
\begin{equation}
    \operatorname{mmse}_{\boldsymbol{X} \mid \boldsymbol{X}_\gamma, S}(\gamma)
    =
    \mathbb{E}_{\boldsymbol{X}, \boldsymbol{X}_\gamma, S}\!\left[
        \left\| \boldsymbol{X} - \mathbb{E}[\boldsymbol{X} \mid \boldsymbol{X}_\gamma, S] \right\|_2^2
    \right].
\end{equation}

We will mainly use two instances, the unconditional and the feature-conditional MMSE:
\begin{equation}
    m_{\varnothing}(\gamma)
    := \operatorname{mmse}_{\boldsymbol{X} \mid \boldsymbol{X}_\gamma}(\gamma),
    \qquad
    m_Y(\gamma)
    := \operatorname{mmse}_{\boldsymbol{X} \mid \boldsymbol{X}_\gamma, Y}(\gamma).
\end{equation}
Since the optimal unconditional denoiser $\hat{\boldsymbol{x}}(\boldsymbol{X}_\gamma, \gamma)$ is also a (possibly suboptimal) feature-conditional denoiser that simply ignores $Y$, the optimal feature-conditional error cannot exceed the unconditional one:
\begin{equation}
    m_Y(\gamma) \le m_{\varnothing}(\gamma)
    \qquad \text{for all } \gamma \ge 0.
\end{equation}
Their difference therefore measures how much the feature $Y$ helps denoise $\boldsymbol{X}$ at SNR $\gamma$, which we call the \emph{single-feature MMSE gap} by conditioning on $Y$:
\begin{equation}
    \Delta_Y(\gamma)
    :=
    m_{\varnothing}(\gamma) - m_Y(\gamma)
    \ge 0.
\end{equation}

\subsection{I-MMSE relation and feature information density}

The I-MMSE identity \citep{immse} states that, for the Gaussian channel above,
\begin{equation}
    \frac{d}{d\gamma} I(\boldsymbol{X}; \boldsymbol{X}_\gamma)
    =
    \frac{1}{2}\, m_\varnothing(\gamma).
\end{equation}
\textcolor{red}{The classical I-MMSE relation describes the mutual information between clean data and a noisy channel output. We next use the same argument to describe the mutual information between a feature and that noisy output.}
Applying the same identity conditionally on $Y$ gives
\begin{equation}
    \frac{d}{d\gamma} I(\boldsymbol{X}; \boldsymbol{X}_\gamma \mid Y)
    =
    \frac{1}{2}\, m_Y(\gamma).
\end{equation}

The channel output $\boldsymbol{X}_\gamma = \sqrt{\gamma}\,\boldsymbol{X} + \boldsymbol{N}$ depends on $Y$ only through $\boldsymbol{X}$, so given $\boldsymbol{X}$ the variables $\boldsymbol{X}_\gamma$ and $Y$ are conditionally independent, i.e. $I(Y; \boldsymbol{X}_\gamma \mid \boldsymbol{X}) = 0$.
The chain rule of mutual information therefore gives
\begin{equation}
    I(\boldsymbol{X}, Y;\, \boldsymbol{X}_\gamma)
    =
    I(\boldsymbol{X}; \boldsymbol{X}_\gamma) + \underbrace{I(Y; \boldsymbol{X}_\gamma \mid \boldsymbol{X})}_{=\,0}
    =
    I(Y; \boldsymbol{X}_\gamma) + I(\boldsymbol{X}; \boldsymbol{X}_\gamma \mid Y),
\end{equation}
so that
\begin{equation}
    I(Y; \boldsymbol{X}_\gamma)
    =
    I(\boldsymbol{X}; \boldsymbol{X}_\gamma) - I(\boldsymbol{X}; \boldsymbol{X}_\gamma \mid Y).
\end{equation}
Differentiating with respect to $\gamma$ yields
\begin{equation}
    D_Y(\gamma)
    :=
    \frac{d}{d\gamma} I(Y; \boldsymbol{X}_\gamma)
    =
    \frac{1}{2}
    \Bigl[
        m_{\varnothing}(\gamma) - m_Y(\gamma)
    \Bigr]
    =
    \frac{1}{2}\Delta_Y(\gamma).
    \label{eq:feature-density-gamma}
\end{equation}
We refer to $D_Y(\gamma)$ as the \emph{feature information density}: it quantifies how the information carried by $Y$ is distributed along the SNR axis.
Integrating it over $\gamma$ recovers the scalar identity $I(Y; \boldsymbol{X}) = \int_0^\infty D_Y(\gamma)\, d\gamma$, exploited by~\citet{mmg} for mutual-information estimation.

More generally, one may parameterize the SNR axis by any smooth, strictly monotone change of variable $\gamma = \gamma(t)$, with derivative $\dot{\gamma}(t) := d\gamma/dt$.
By the chain rule, the feature information density on the new $t$-axis is
\begin{equation}
    \frac{d}{dt} I(Y; \boldsymbol{X}_{\gamma(t)})
    =
    D_Y(\gamma(t))\, \dot{\gamma}(t).
\end{equation}
 A particularly useful instance is the logarithmic SNR axis $t = \log\gamma$, for which $\dot{\gamma}(t) = \gamma$, giving
\begin{equation}
    D_Y^{(\log\gamma)}(\log\gamma)
    :=
    \frac{d}{d \log \gamma} I(Y; \boldsymbol{X}_\gamma)
    =
    \gamma\, D_Y(\gamma)
    =
    \frac{1}{2}\gamma\, \Delta_Y(\gamma).
    \label{eq:feature-density-logsnr}
\end{equation}
Here we keep the same base symbol $D$ and indicate the axis through the superscript.
This representation balances contributions across many decades of SNR and will be useful when comparing densities across different time parameterizations.

\subsection{Diffusion and flow-matching parameterization}

We now instantiate the framework above with diffusion or Gaussian flow-matching models~\citep{cfm}. 
The variance of the Gaussian channel $\boldsymbol{X}_\gamma = \sqrt{\gamma}\,\boldsymbol{X} + \boldsymbol{N}$ grows unboundedly with $\gamma$; it can however be brought back to a bounded range by a $\gamma$-dependent normalization, which is invertible and therefore loses no information.
Through the well-known linear relations among the clean-data, noise, and velocity targets, the MMSE ($\boldsymbol{x}$-prediction loss) transforms equivalently into the standard diffusion or flow-matching training objective, and the optimal denoiser transforms equivalently into the optimal diffusion or flow-matching model.

Unless otherwise specified, we use \textcolor{red}{Gaussian Flow Matching with independent coupling}\footnote{\textcolor{red}{For a non-independent endpoint coupling, GFM is no longer equivalent to a Gaussian channel; that case is outside the scope of this work.}} and the linear interpolation
\begin{equation}
    \boldsymbol{X}_t = (1-t)\, \boldsymbol{N} + t\, \boldsymbol{X},
    \qquad t \in [0,1],
\end{equation}
and velocity target $\boldsymbol{V}_t = \boldsymbol{X} - \boldsymbol{N}$.
Identifying $\boldsymbol{X}_t/(1-t) = \boldsymbol{N} + \sqrt{\gamma(t)}\, \boldsymbol{X} \stackrel{d}{=} \boldsymbol{X}_{\gamma(t)}$ gives the SNR $\gamma(t) = \left(\frac{t}{1-t}\right)^2$.
The optimal flow-matching velocity predictor and the optimal denoiser are related by
\begin{equation}
    \hat{\boldsymbol{x}}^\star(\boldsymbol{X}_t, t) = \boldsymbol{X}_t + (1-t)\,\hat{\boldsymbol{v}}^\star(\boldsymbol{X}_t, t),
\end{equation}
and the MMSE is recovered from the minimum velocity loss as
\begin{equation}
    m_\varnothing(\gamma(t))
    = (1-t)^2 \min_{\hat{\boldsymbol{v}}} \mathbb{E}_{\boldsymbol{X}, \boldsymbol{N}}\!\left[\|\boldsymbol{V}_t - \hat{\boldsymbol{v}}(\boldsymbol{X}_t, t)\|_2^2\right],
    \label{eq:loss-mmse}
\end{equation}
with $m_Y(\gamma(t))$ obtained analogously from the feature-conditional velocity loss.
The single-feature MMSE gap $\Delta_Y(\gamma(t)) = m_\varnothing(\gamma(t)) - m_Y(\gamma(t))$ is therefore obtained directly from the difference of optimal unconditional and feature-conditional flow-matching losses, without altering the training objective.

\subsection{Practical estimators}

Given trained unconditional and feature-conditional denoisers, we estimate the single-feature MMSE gap $\hat{\Delta}_Y(\gamma)$ by Monte Carlo over test samples and form the empirical density $\hat{D}_Y(\gamma) = \tfrac{1}{2}\hat{\Delta}_Y(\gamma)$ as in Eq.~\eqref{eq:feature-density-gamma}.
However, several equivalent estimator variants---posterior-mean, Fisher, score-difference, and likelihood-score forms---are derived and compared in Appendix~\ref{app:estimators}.

In practice, the raw single-feature MMSE gap is reliable only over a finite SNR range.
At very high SNR, denoising errors are small and gradients are weak, so a finite-capacity denoiser can underfit the clean-data limit.
This causes the empirical MMSE to be overestimated in the high-SNR tail, and a direct numerical integral of the gap can therefore be biased upward or even fail to converge.
Inspired by analytical properties of MMSE curves \citep{mmse_dim}, we address this by fitting the \emph{single-feature MMSE gap itself} with a smooth double-pole Lorentzian family on a reliable interval and extrapolating the fitted tail analytically. This simple fit matches the empirical gap surprisingly well, detailed in Appendix~\ref{app:fitting}.

\section{Chained Information Decomposition}

The preceding section defines the information density of a \emph{single} feature $Y$ through the single-feature MMSE gap $m_{\varnothing}(\gamma)-m_Y(\gamma)$.
Many of our experiments, however, concern a hierarchy of related features, such as frequency bands or semantic--structural image cues.
In this setting, computing each single-feature density $D_{Y_k}(\gamma)$ separately can double-count information shared across feature levels.

To isolate the incremental contribution of each feature, we therefore group the first $k$ features into a cumulative bundle $\mathcal{Y}_{\le k}$ and consider a nested group of conditioning bundles
\[
    \mathcal{Y}_{\le 0} = \varnothing \subseteq \mathcal{Y}_{\le 1} \subseteq \cdots \subseteq \mathcal{Y}_{\le K},
\]
for example $\mathcal{Y}_{\le 1} = \text{class}$, $\mathcal{Y}_{\le 2} = (\text{class}, \text{mask})$, and $\mathcal{Y}_{\le 3} = (\text{class}, \text{mask}, \text{canny})$.
The chained construction then measures the incremental information contributed by $Y_k$ relative to $\mathcal Y_{\le k-1}$.

Writing $m_k(\gamma) := \operatorname{mmse}_{\boldsymbol{X} \mid \boldsymbol{X}_\gamma, \mathcal{Y}_{\le k}}(\gamma)$, we define the \emph{chained information density} of $Y_k$ as the incremental information of $\mathcal{Y}_{\le k}$ beyond $\mathcal{Y}_{\le k-1}$. \textcolor{red}{By the same conditional I-MMSE argument as in Section~3.2,}
\begin{equation}
\textcolor{red}{\begin{aligned}
    D_{k \mid k-1}^{(\log\gamma)}(\log\gamma)
    &:= \frac{d}{d\log\gamma} I(\boldsymbol{X}_\gamma;Y_k\mid \mathcal{Y}_{\le k-1}) \\
    &= \tfrac{1}{2}\gamma\,\bigl[m_{k-1}(\gamma)-m_k(\gamma)\bigr] \ge 0,
    \qquad k = 1, \ldots, K.
\end{aligned}}
    \label{eq:chained-density-def}
\end{equation}
\textcolor{red}{We denote the bracketed difference by the chained MMSE gap, $\Delta^{\mathrm{chain}}_k(\gamma):=m_{k-1}(\gamma)-m_k(\gamma)$. Thus, as in Section~3.2, integrating the density gives the corresponding conditional-information increment.}
Figure~\ref{fig:fid-overview} summarizes the operational core of our framework on a concrete visual hierarchy.
At each SNR $\gamma$, we compare an unconditional denoiser to denoisers conditioned on successively richer feature bundles
$\mathcal{Y}_{\le 1} = \text{class} \subseteq \mathcal{Y}_{\le 2} = (\text{class},\text{mask}) \subseteq \mathcal{Y}_{\le 3} = (\text{class},\text{mask},\text{Canny})$;
the chained MMSE gap $\Delta^{\mathrm{chain}}_k(\gamma) = m_{k-1}(\gamma) - m_k(\gamma)$ in Eq.~\eqref{eq:chained-density-def} converts these denoiser-loss differences into per-level information densities $D^{(\log\gamma)}_{k\mid k-1}(\log\gamma)$ along the log-SNR axis.
The peak location and width of each $D^{(\log\gamma)}_{k\mid k-1}$ localize \emph{when} along the diffusion trajectory the corresponding feature is most informative, and the ordering of these peaks gives an empirical signature of how the chosen representation organizes its feature information dynamics.

\begin{figure}[t]
  \centering
  \includegraphics[width=\textwidth]{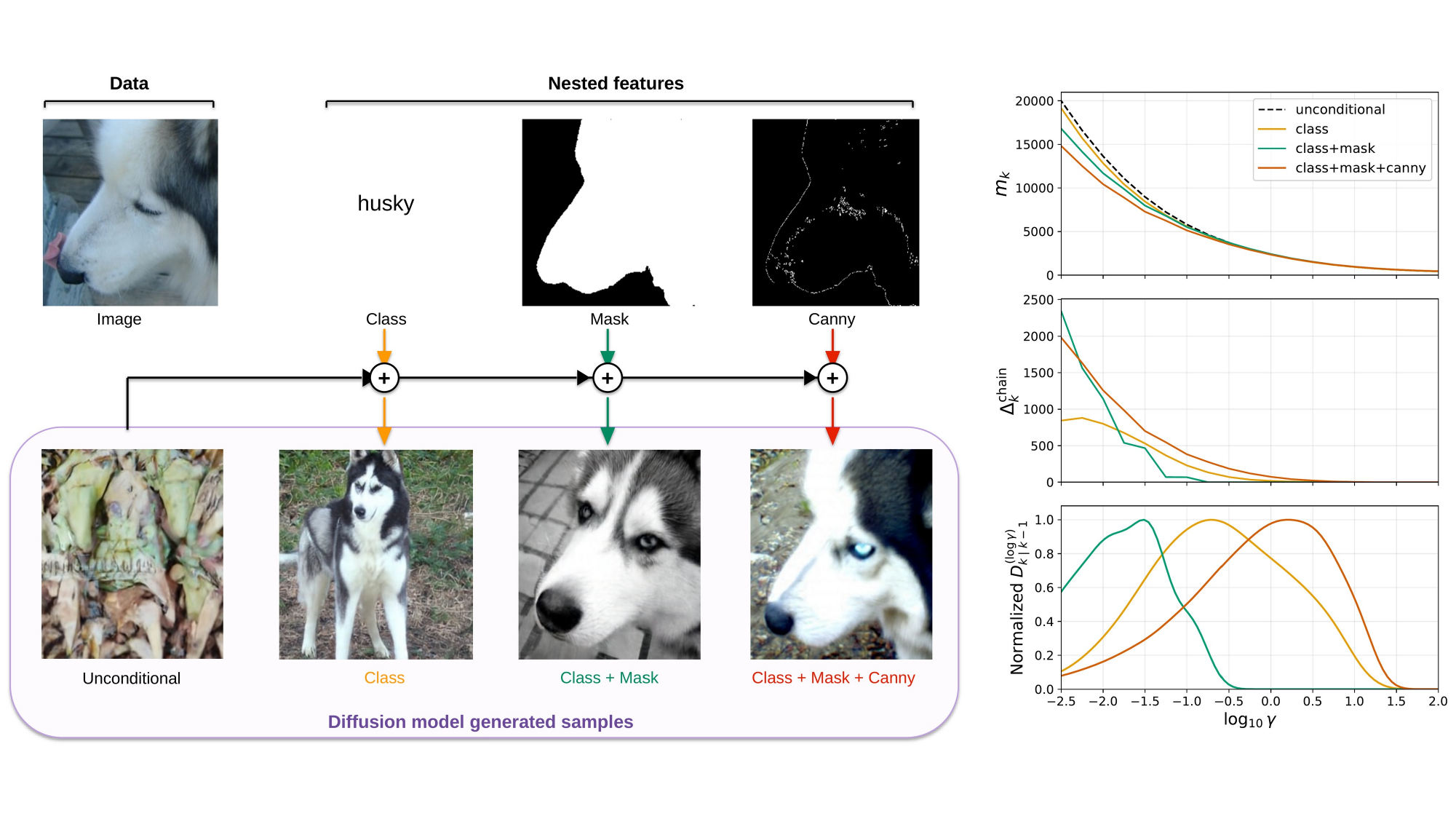}
  \caption{%
    \textcolor{red}{\textbf{Pixel-space example of chained feature information dynamics.}}
    \emph{Top:} a clean image $\boldsymbol{X}$ paired with a nested feature hierarchy of increasing spatial specificity --- class label ($\mathcal{Y}_{\le 1}$), segmentation mask ($\mathcal{Y}_{\le 2}$), and Canny edges ($\mathcal{Y}_{\le 3}$). Conditioning on the cumulative bundle progressively constrains the diffusion model's generated samples (\emph{bottom row}): the unconditional model produces an arbitrary natural image; class fixes semantic identity; adding mask fixes object-level spatial support and pose; adding Canny additionally fixes local boundary structure.
    \emph{Right:} at each SNR $\gamma$, conditioning on a richer bundle reduces $m_k(\gamma)$ (\emph{top right}); the chained MMSE gap $\Delta^{\mathrm{chain}}_k$ in Eq.~\eqref{eq:chained-density-def} (\emph{middle right}) yields the per-level information density $D^{(\log\gamma)}_{k\mid k-1}$ (\emph{bottom right}), which localizes \emph{where} along the trajectory each feature contributes its incremental information. Curves peaking at distinct log-SNR values are the operational signature of a hierarchically organized denoising trajectory.
  }
  \label{fig:fid-overview}
\end{figure}
\textcolor{red}{The chained attribution depends on the chosen feature order because each increment is conditioned on the preceding bundle. We use class$\,\to\,$mask$\,\to\,$Canny as the canonical semantic-to-local order; the implications of this choice are discussed in the \hyperref[sec:discussion:limitations]{Limitations}.}


\section{Experiments}
\label{sec:exp}

\subsection{Chained decomposition reveals spectral AR in pixel diffusion}
\label{sec:exp:spectral-hierarchy}

Pixel-space diffusion is often described as coarse-to-fine: global structure appears at high noise levels, while fine texture is resolved near the clean-data end of the trajectory.
A precise version of this claim is spectral autoregression, where low spatial frequencies are resolved before high spatial frequencies.
Feature information dynamics allows us to test this claim quantitatively, but only after separating incremental frequency information from information shared across bands.

\begin{figure}[t]
  \centering
  \begin{subfigure}[t]{0.46\textwidth}
    \centering
    \includegraphics[width=\linewidth]{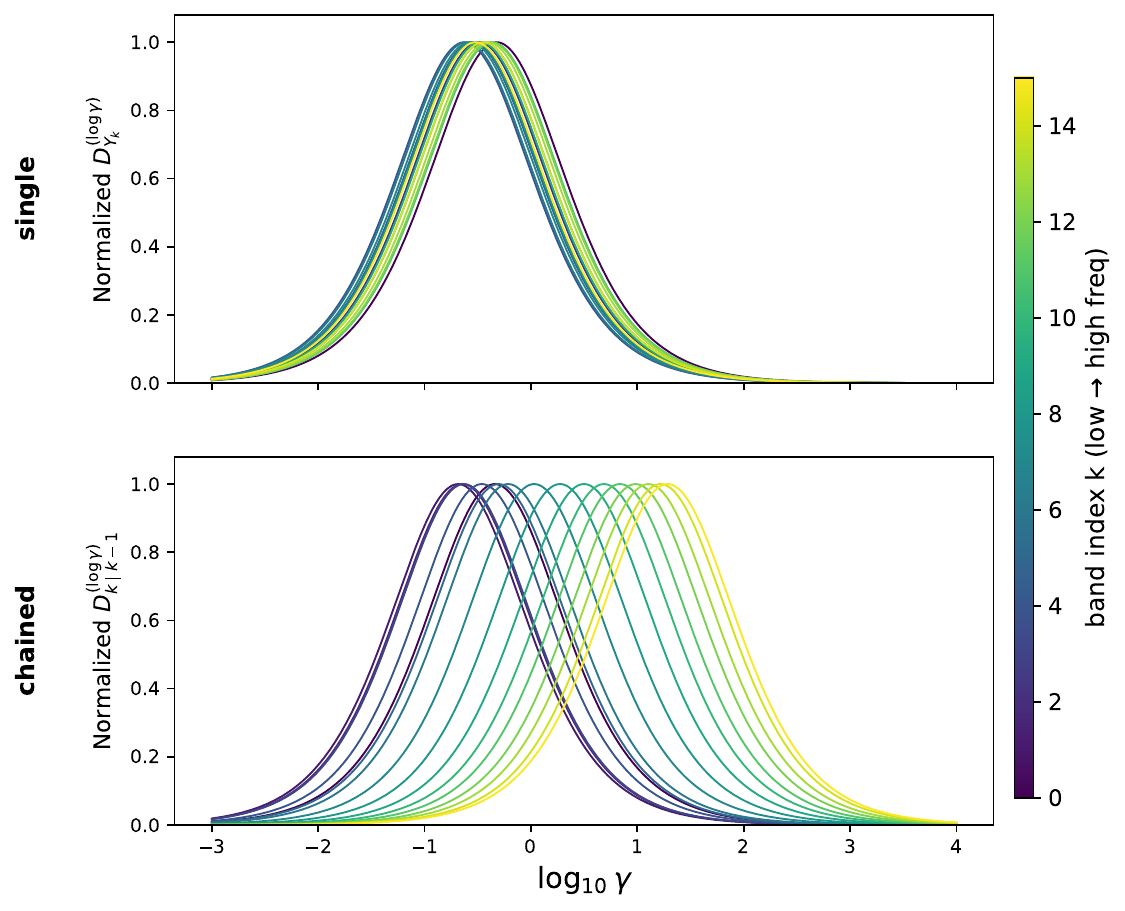}
    \caption{MNIST: single-band (top) vs.\ chained (bottom) conditioning.}
    \label{fig:chained-order-comparison}
  \end{subfigure}
  \hfill
  \begin{subfigure}[t]{0.46\textwidth}
    \centering
    \includegraphics[width=\linewidth]{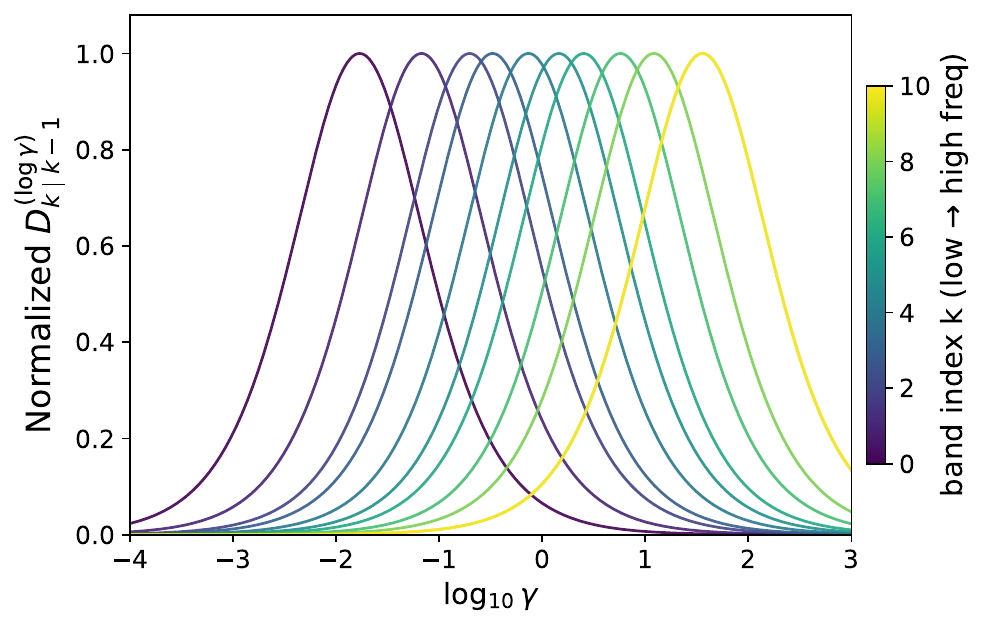}
    \caption{CIFAR-10: chained spectral progression.}
    \label{fig:cifar10-freq}
  \end{subfigure}
  \caption{%
    \textbf{Spectral information densities.}
    In the top row of (a), single-band conditioning produces densities $D^{(\log\gamma)}_{Y_k}$ that overlap heavily because shared low-frequency information is counted in every band; in the bottom row of (a) and in (b), forward chaining removes information already explained by coarser bands, revealing a clean coarse-to-fine progression along $\log\gamma$. Color: low (purple) to high (yellow) frequency band index $k$.
  }
  \label{fig:spectral-densities}
\end{figure}

We start with a controlled MNIST frequency hierarchy.
Each image is decomposed into ordered frequency bands $Y_0,\ldots,Y_K$, from coarse to fine.
We compare two decompositions of the same bands.
In the \emph{single-band} setting, each band $Y_k$ is conditioned independently.
In the \emph{forward-chained} setting, the condition is the cumulative bundle $\mathcal{Y}_{\le k}=(Y_0,\ldots,Y_k)$.
Single-band MMSE gaps and chained MMSE gaps are converted to log-SNR information densities using Eq.~\eqref{eq:feature-density-logsnr} and Eq.~\eqref{eq:chained-density-def}, respectively.
Figure~\ref{fig:chained-order-comparison} compares the resulting densities: single-band conditioning produces overlapping densities because shared information is counted repeatedly, while forward chaining removes information already explained by coarser bands and reveals a clean coarse-to-fine progression.

We then apply the same chained construction to CIFAR-10.
Full training and evaluation details are given in Appendix~\ref{app:spectral-hierarchy-details}.
Figure~\ref{fig:cifar10-freq} shows a monotone spectral progression.
Low-frequency bands concentrate their information at lower log-SNR, whereas higher-frequency bands shift toward cleaner denoising states.
Thus pixel-space diffusion approximately follows spectral autoregression in a quantitative sense: feature information density identifies not only that coarse frequencies precede fine ones, but also where along the SNR axis each incremental frequency band is resolved.

\subsection{Representations reshape hierarchical information dynamics}
\label{sec:exp:representation-analysis}

The spectral experiment confirms a coarse-to-fine ordering for frequency bands, but it does not by itself say how semantic, structural, and textural information are generated in modern image diffusion models.
We now apply chained feature information dynamics to a general visual feature hierarchy --- class labels (global semantics, high-level), segmentation masks (object-level spatial support, mid-level), and Canny edges (local boundary structure, low-level) --- and ask, with the same data and the same conditioning chain, how four representations reorganize the noise levels at which each level is generated.

We compare four representations chosen to span the current latent-diffusion design landscape, each paired with the standard generator backbone reported in its original work:
\emph{(i)}~\textbf{pixel} space, the unprocessed image grid, with the JiT-L/16 backbone~\citep{jit};
\emph{(ii)}~\textbf{SDVAE}~\citep{ldm}, a continuous VAE latent that balances computational efficiency and reconstruction quality, with the SiT-XL/2 backbone~\citep{sit};
\emph{(iii)}~\textbf{VAVAE}~\citep{vavae}, which augments a VAE-style latent with vision-foundation-model alignment to accelerate diffusion convergence, with the LightningDiT-XL/1 backbone~\citep{vavae};
\emph{(iv)}~\textbf{RAE}~\citep{rae}, a representation autoencoder built directly on DINOv2 features, with the DiTDH-XL backbone~\citep{rae}.
This selection is intentionally small but covers the dominant latent-diffusion design axis (no encoder $\to$ reconstruction-driven $\to$ VFM-aligned $\to$ pure VFM).
We use a SAM~3.1-validated subset of ImageNet-256: for each ImageNet class we run SAM~3.1~\citep{sam3} on the training images and retain only those classes for which the mask success rate exceeds $95\%$ (485 classes); Canny edges are computed inside the retained mask. The paired \{image, class, mask, Canny\} dataset is shared verbatim across all representations.

\textbf{Roadmap.} Section~\ref{sec:exp:training-speed} runs a tightly controlled training-speed test that establishes the convergence ordering pixel $<$ SDVAE $<$ VAVAE $<$ RAE under matched optimization recipes, and points out that \emph{why} representations differ in training speed remains an open problem. Section~\ref{sec:exp:hierarchy-diagnostic} then uses the chained feature information dynamics, together with the hierarchical latent-variable prior, as a diagnostic for this gap and reads the per-representation panoramas through that lens.

\subsubsection{Training-speed test under matched recipes}
\label{sec:exp:training-speed}

Existing comparisons of training speed across image diffusion representations are mostly based on published FID curves that differ in \emph{both} generator architecture and training recipe (optimizer, time-shift schedule, sampler, learning-rate warmup, number of FID samples, etc.).
This makes it difficult to attribute observed differences to the representation alone rather than to recipe choices.
We therefore design a representation-only training-speed test in which the recipe variables that most affect early FID --- training noise shift $\alpha_t$, sampling noise shift $\alpha_s$, and learning rate --- are first selected on a smaller LDiT-B model with a fixed grid (Appendix~\ref{app:rep-exp-details}), and then frozen across all four XL trajectories.
All four representations share the same number of tokens (256), the same model (LDiT-XL), the same sampler (Euler with 50 NFE, unconditional inference).
What remains different is exactly what defines each representation: encoder family, native patch size, channel count, and prediction parameterization ($\boldsymbol{x}$-prediction~\citep{jit} for pixel/RAE, velocity for SDVAE/VAVAE), as required by the standard recipe of each space.

\begin{figure}[t]
  \centering
  \includegraphics[width=0.62\textwidth]{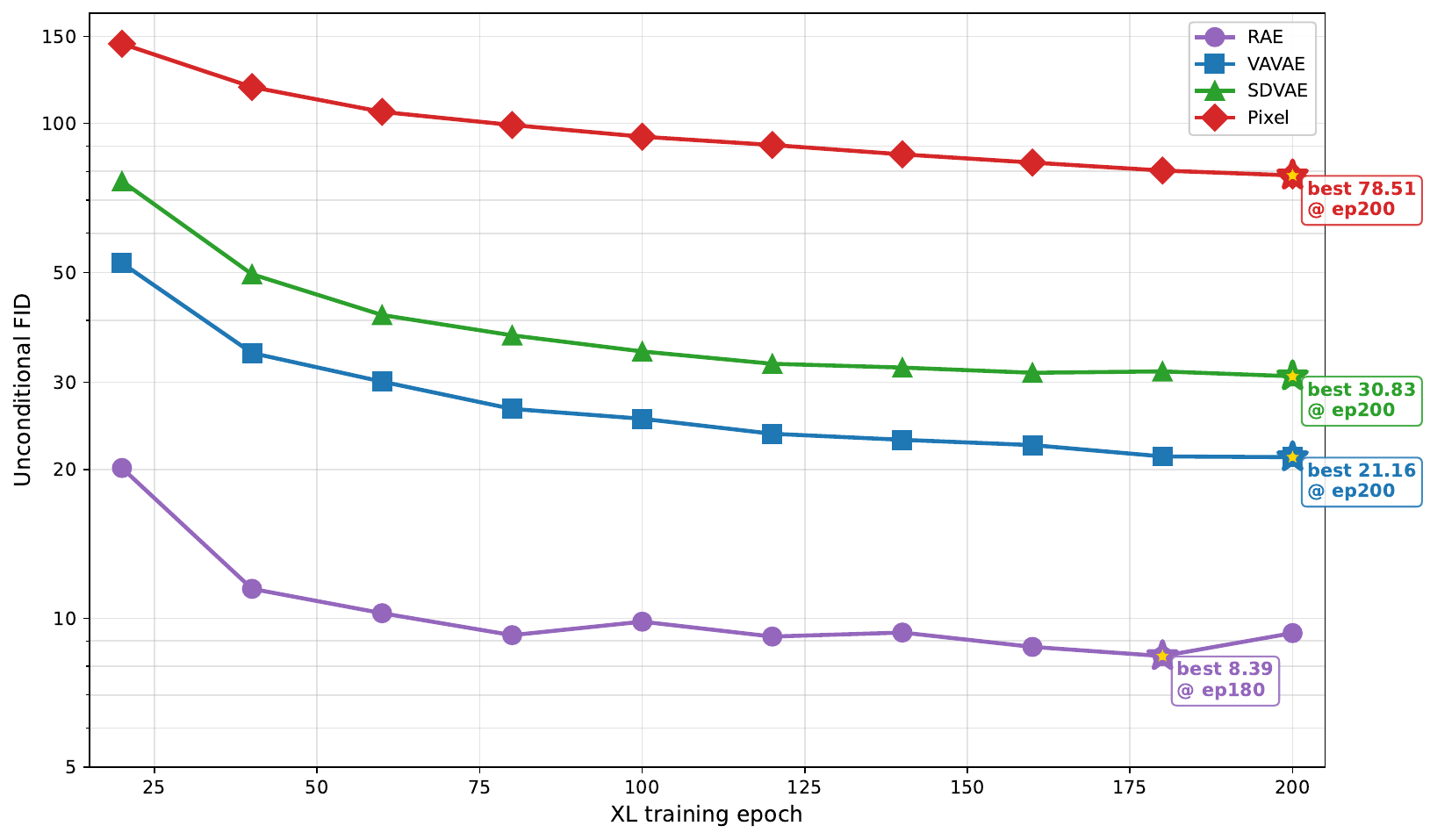}
  \caption{%
    \textbf{Unconditional FID trajectories under a matched recipe.}
    Best FID reached over 200 XL epochs: RAE $8.39$ (ep.~180) $<$ VAVAE $21.16$ $<$ SDVAE $30.83$ $<$ pixel $78.51$ (all at ep.~200). Hyperparameters $(\alpha_t, \alpha_s, \mathrm{lr})$ are selected on a smaller LDiT-B model and then frozen across the XL trajectories.
  }
  \label{fig:rep-uncond-fid}
\end{figure}

Figure~\ref{fig:rep-uncond-fid} shows the resulting FID trajectories.
The four representations do not merely reach different final FID values; they follow visibly distinct convergence curves whose ordering is consistent across the whole 200-epoch budget:
$$\textbf{pixel} \;<\; \textbf{SDVAE} \;<\; \textbf{VAVAE} \;<\; \textbf{RAE},$$
with best FID $78.51$ (pixel, ep.~200) $\to$ $30.83$ (SDVAE, ep.~200) $\to$ $21.16$ (VAVAE, ep.~200) $\to$ $8.39$ (RAE, ep.~180).
This ordering is robust to the recipe variables we explicitly controlled and is consistent with the qualitative ranking reported in the original VAVAE and RAE papers.
What this experiment does \emph{not} explain is the mechanism: it is an open problem why a particular representation makes diffusion training easier or harder under a matched optimization recipe.

\subsubsection{Hierarchical information-dynamics diagnostic}
\label{sec:exp:hierarchy-diagnostic}

We now use the chained feature information dynamics, together with the hierarchical latent-variable prior implicit in the class$\,\to\,$mask$\,\to\,$Canny chain, as a diagnostic for the convergence-speed ordering established above.
For every representation, we initialize each conditioning stage from the publicly released pretrained checkpoint of the corresponding backbone~\citep{jit,sit,vavae,rae} and then fine-tune for the new condition following official pipelines with necessary modifications, so that each $m_k(\gamma)$ is evaluated near the best diffusion-loss / MMSE floor reachable in that space rather than at an early-stage training optimum. The illustration in Figure~\ref{fig:fid-overview} visualizes the pixel instance of this chain: the four sample columns are produced by phase-wise fine-tuning a JiT-L/16~\citep{jit} model through the unconditional, class, class$+$mask, and class$+$mask$+$Canny modes, and the right-hand panels are the corresponding chained MMSE curves and per-level densities.
The appendix panoramas (Figure~\ref{fig:rep-hierarchy-panorama}) give the analogous diagnostic for SDVAE, VAVAE, and RAE.

Detailed conditioning protocol, token counts, and per-stage hyperparameters are in Appendix~\ref{app:rep-exp-details}; the full per-representation panoramas (MMSE curves, chained MMSE gaps, log-SNR information densities, and cumulative information) are in Figure~\ref{fig:rep-hierarchy-panorama}.

\begin{figure}[t]
  \centering
  \includegraphics[width=0.60\textwidth]{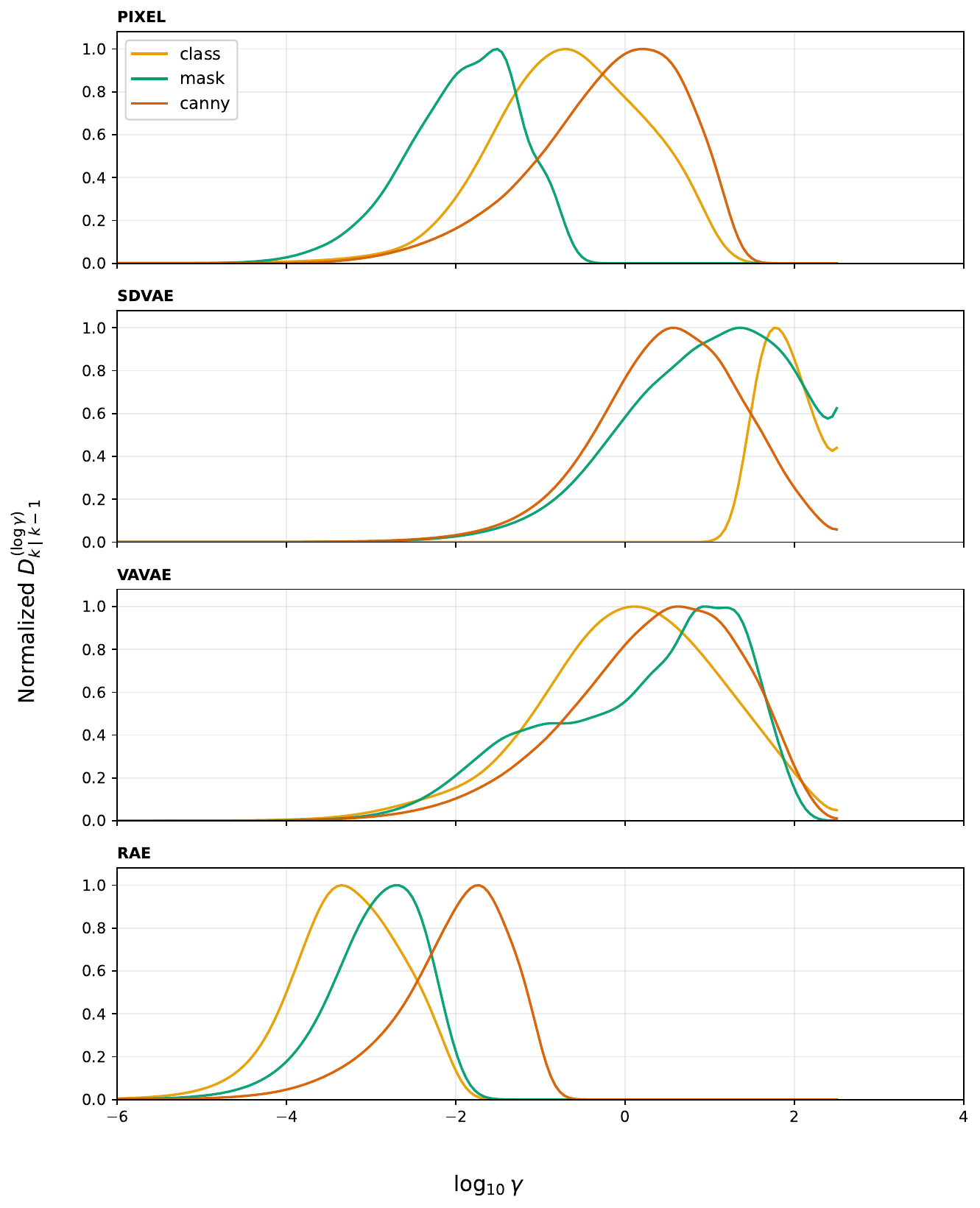}
  \caption{%
    \textbf{Layer-wise feature information dynamics across representations.}
    Per-feature log-SNR information densities $D^{(\log\gamma)}_{k\mid k-1}$ of the class, mask, and Canny increments for pixel, SDVAE, VAVAE, and RAE diffusion. Full panoramas in Figure~\ref{fig:rep-hierarchy-panorama}.
  }
  \label{fig:rep-hierarchy}
\end{figure}

Figure~\ref{fig:rep-hierarchy} reveals that the same conditioning chain produces qualitatively different feature information dynamics across representations, and that these differences track the convergence-speed ordering of Section~\ref{sec:exp:training-speed}.

\begin{itemize}
  \item \textbf{Pixel space (slowest convergence) inverts the classical hierarchy.}
    The mask density peaks at lower log-SNR than the class density, so mask information enters earlier and more strongly than class identity. Pixel diffusion does \emph{not} resolve semantics first.
  \item \textbf{VAVAE (faster convergence) recovers the semantic-first half.}
    The class increment is now prioritized, with its peak located at lower log-SNR than the mask peak, but mask and Canny still overlap rather than separate cleanly along $\log\gamma$.
  \item \textbf{Only RAE (fastest convergence) achieves a strict class$\,\to\,$mask$\,\to\,$Canny ordering.}
    Its three density peaks land in textbook high$\,\to\,$mid$\,\to\,$low order along the log-SNR axis, matching the classical visual hierarchy.
\end{itemize}

We therefore observe a correlation between how cleanly a representation hierarchically orders its feature densities and how quickly its diffusion model converges, without claiming that this ordering causes faster convergence.
\begin{table}[t]
\color{red}
\centering
\small
\begin{tabular}{@{}llrrr@{}}
\toprule
Model / checkpoint & Conditioning order & Class & Mask & Canny \\
\midrule
RAE-S, 20 epochs & class--mask--Canny & $-5.00$ & $-3.50$ & $-2.00$ \\
RAE-S, 40 epochs & class--mask--Canny & $-3.75$ & $-2.75$ & $-1.50$ \\
RAE-S, 60 epochs & class--mask--Canny & $-3.75$ & $-2.75$ & $-1.75$ \\
RAE-S, 80 epochs & class--mask--Canny & $-3.75$ & $-2.75$ & $-1.75$ \\
RAE/DiTDH-XL & class--mask--Canny & $-3.75$ & $-2.75$ & $-1.75$ \\
RAE/DiTDH-XL & mask--class--Canny & $-4.00$ & $-3.25$ & $-1.75$ \\
\bottomrule
\end{tabular}
\vspace{4pt}
\caption{\color{red}Peak log-SNRs for the model-size, checkpoint, and chain-order checks. RAE-S has approximately 132M parameters and the DiTDH-XL model approximately 842M.}
\label{tab:camera-ready-stability}
\end{table}

\textbf{The central message is that representation choice acts as a controllable knob over feature information dynamics:} the same conditioning chain, evaluated against the same images and masks, yields representation-specific orderings that are invisible from FID alone.


\section{Discussion}
\label{sec:discussion}

\subsection{\textcolor{red}{Generation order in vision and language}}
\label{sec:discussion:generation-order}
The hierarchical-ordering view of diffusion that emerges here connects to a long line of work on \emph{generation order} in non-diffusion generative models.
On the vision side, next-scale prediction (VAR)~\citep{var} and masked-token refinement (MaskGIT)~\citep{maskgit} explicitly hard-code a coarse-to-fine generation order, and report that this order improves both sample quality and parallelism over standard left-to-right tokenization.
On the language side, XLNet~\citep{xlnet} permutes the autoregressive factorization to expose bidirectional context, while insertion-based sequence models~\citep{insertion_transformer} expose generation order itself as a controllable design choice with quality--parallelism trade-offs.
A common message is that, even when the underlying joint distribution is fixed, \emph{the order in which features are introduced} changes the optimization geometry seen by the model.
Our framework reads diffusion in the same vocabulary: chaining the conditioning bundle along a feature hierarchy makes the implicit generation order along the noise schedule directly observable as the placement of $D^{(\log\gamma)}_{k\mid k-1}$ peaks.
The cross-representation finding in Section~\ref{sec:exp:hierarchy-diagnostic} is therefore best understood as an order-discovery result: pixel, SDVAE, VAVAE, and RAE realize different implicit orders for the same data and the same conditioning chain, and the cleanest order coincides with the fastest training.

\subsection{\textcolor{red}{Limitations}}
\label{sec:discussion:limitations}
\textcolor{red}{Our framework has several limitations. Two are especially consequential for interpreting the empirical profiles and are examined directly below.}

\paragraph{\textcolor{red}{Approximation by trained denoisers.}}
\textcolor{red}{The estimator uses trained denoisers to approximate Bayes-optimal MMSE terms, so the measured gaps can depend on optimization quality as well as on the underlying information. The order, checkpoint, and model-size checks in Table~\ref{tab:camera-ready-stability} provide a direct check: the alternative mask$\,\to\,$class$\,\to\,$Canny order shifts the class and mask peaks but preserves their separation on RAE, while approximately 132M-parameter RAE-S checkpoints and the approximately 842M-parameter DiTDH-XL model show the same late-stage ordering and peak locations. Appendix~\ref{app:hierarchy-robustness} also records that pretrained-checkpoint initialization and training from scratch can lead to different measured conclusions.}

\paragraph{\textcolor{red}{Dependence on feature order.}}
\textcolor{red}{The chained attribution is order-dependent because each increment is conditioned on the preceding feature bundle; it is therefore not an order-independent assignment of information. We use class$\,\to\,$mask$\,\to\,$Canny as a canonical semantic-to-local order, and the alternative-order result in Table~\ref{tab:camera-ready-stability} tests how this choice affects the profiles.}

\paragraph{\textcolor{red}{Other limitations.}}
\textcolor{red}{The cross-representation comparison is limited by computational budget to four representations and one conditioning chain. The mutual-information quantities are distribution-level averages, so the profiles need not describe every individual image. The density is a descriptive diagnostic, and the observed association between hierarchical ordering and convergence speed does not establish a causal mechanism. The analysis also assumes an additive Gaussian channel with independent coupling; settings with a non-independent endpoint coupling require an extension beyond the present scope. Finally, applying the direct chained estimator to a larger hierarchy requires additional cumulative conditional models, which can become a practical bottleneck.}

\subsection{\textcolor{red}{Possible uses}}
\label{sec:discussion:possible-uses}
\textcolor{red}{The diagnostic suggests several directions for application-driven study. First, before committing to a full diffusion-model training run, researchers could compare representations by asking whether semantic, geometric, and local information are separated along the SNR axis; in our experiments, cleaner separation accompanies faster convergence, so the profiles may serve as a representation-selection signal without being treated as a proven predictor. Second, the locations and widths of the density peaks could inform candidate timestep schedules or conditioning orders by identifying the SNR ranges in which particular features are most active. Third, conditional samplers could use feature- or SNR-dependent guidance strengths instead of applying the same guidance throughout the trajectory. Finally, comparing these profiles across architectures or representations could help diagnose where information becomes entangled or remains unresolved, and could motivate representation or conditioning designs targeted at those stages.}


\section{Conclusion}
\label{sec:conclusion}

We presented \emph{feature information dynamics}, a quantitative framework that localizes when information about a feature enters the diffusion trajectory by combining the I-MMSE identity with unconditional and feature-conditional denoising losses.
The chained decomposition extends this to nested feature hierarchies and isolates the incremental contribution of each level.
Empirically, pixel diffusion follows spectral autoregression in the chained sense, and applying the same conditioning chain (class$\,\to\,$mask$\,\to\,$Canny) across pixel, SDVAE, VAVAE, and RAE diffusion produces qualitatively different feature information dynamics: only RAE achieves a strict class$\,\to\,$mask$\,\to\,$Canny ordering of density peaks, while pixel and SDVAE deviate.
This recasts representation choice as a controllable knob over feature information dynamics and possibly diffusion training efficiency, an interplay between representation and generation to be explored.

\begin{ack}
\textcolor{red}{We gratefully acknowledge the support of Hangzhou West Lake Pearl Program; Westlake University Research Center for Industries of the Future; and Westlake University Center for High-performance Computing.}

\textcolor{red}{The authors declare no competing interests.}
\end{ack}

\bibliographystyle{plainnat}
\bibliography{references}

\clearpage
\appendix

\section{Equivalent Estimators}
\label{app:estimators}

The feature information density $D_Y(\gamma)$ defined in Eq.~\eqref{eq:feature-density-gamma} admits several equivalent expressions, each corresponding to a distinct practical estimator.
We begin by recalling the \emph{Fisher information} of the channel output, defined as the expected squared norm of the score:
\begin{equation}
    J(X_\gamma) := \E_{\boldsymbol{X}_\gamma}\!\left[\|\nabla\log p(\boldsymbol{X}_\gamma)\|^2\right],
    \qquad
    J(X_\gamma \mid Y) := \E_{\boldsymbol{X}_\gamma, Y}\!\left[\|\nabla\log p(\boldsymbol{X}_\gamma \mid Y)\|^2\right],
\end{equation}
where $\nabla \log p(x_\gamma)$ denotes the gradient of $\log p$ with respect to its argument $x_\gamma$.

\begin{proposition}[Equivalent Estimators]
\label{prop:equivalence-app}
Under standard regularity (finite second moments, smooth densities), the feature information density admits the equivalent forms
\begin{align}
    D_Y(\gamma)
    &= \tfrac{1}{2}\bigl[m_\varnothing(\gamma) - m_Y(\gamma)\bigr]
    && \text{(single-feature MMSE gap)}
    \label{eq:est-mmse-app} \\
    &= \tfrac{1}{2}\,\E_{\boldsymbol{X}_\gamma, Y}\!\left[\bigl\|\hat{\boldsymbol{x}}^\star(\boldsymbol{X}_\gamma, \gamma, Y) - \hat{\boldsymbol{x}}^\star(\boldsymbol{X}_\gamma, \gamma)\bigr\|^2\right]
    && \text{(denoiser gap)}
    \label{eq:est-denoiser-app} \\
    &= \tfrac{1}{2\gamma}\,\E_{\boldsymbol{X}_\gamma, Y}\!\left[\bigl\|\nabla\log p(\boldsymbol{X}_\gamma \mid Y) - \nabla\log p(\boldsymbol{X}_\gamma)\bigr\|^2\right]
    && \text{(score gap)}
    \label{eq:est-score-app} \\
    &= \tfrac{1}{2\gamma}\bigl[J(X_\gamma \mid Y) - J(X_\gamma)\bigr]
    && \text{(Fisher info gap)}
    \label{eq:est-fisher-app} \\
    &= \tfrac{1}{2\gamma}\,\E_{\boldsymbol{X}_\gamma, Y}\!\left[\bigl\|\nabla\log p(Y \mid \boldsymbol{X}_\gamma)\bigr\|^2\right]
    && \text{(likelihood score).}
    \label{eq:est-lscore-app}
\end{align}
\end{proposition}

\begin{proof}
\noindent\textbf{Denoiser gap, Eq.~\eqref{eq:est-denoiser-app}.}\quad
This proof was first given by \citet{ortho}, and we rephrase here for completeness. At each $\gamma$, by the definitions of $m_\varnothing$ and $m_Y$,
\begin{equation*}
    m_\varnothing(\gamma) - m_Y(\gamma)
    = \E_{\boldsymbol{X}, \boldsymbol{X}_\gamma, Y}\!\left[\bigl\|\boldsymbol{X} - \hat{\boldsymbol{x}}^\star(\boldsymbol{X}_\gamma, \gamma)\bigr\|^2 - \bigl\|\boldsymbol{X} - \hat{\boldsymbol{x}}^\star(\boldsymbol{X}_\gamma, \gamma, Y)\bigr\|^2\right].
\end{equation*}
Apply the algebraic identity $\|A\|^2 - \|B\|^2 = \|A - B\|^2 + 2\, B \cdot (A - B)$ with $A = \boldsymbol{X} - \hat{\boldsymbol{x}}^\star(\boldsymbol{X}_\gamma, \gamma)$ and $B = \boldsymbol{X} - \hat{\boldsymbol{x}}^\star(\boldsymbol{X}_\gamma, \gamma, Y)$, so that $A - B = \hat{\boldsymbol{x}}^\star(\boldsymbol{X}_\gamma, \gamma, Y) - \hat{\boldsymbol{x}}^\star(\boldsymbol{X}_\gamma, \gamma)$.
The cross term $\E[2 B \cdot (A - B)]$ vanishes: conditioning on $(\boldsymbol{X}_\gamma, Y)$, the factor $A - B$ is deterministic, while $\E[\boldsymbol{X} - \hat{\boldsymbol{x}}^\star(\boldsymbol{X}_\gamma, \gamma, Y) \mid \boldsymbol{X}_\gamma, Y] = 0$ by definition of the optimal denoiser.
Hence $m_\varnothing(\gamma) - m_Y(\gamma) = \E\|\hat{\boldsymbol{x}}^\star(\boldsymbol{X}_\gamma, \gamma, Y) - \hat{\boldsymbol{x}}^\star(\boldsymbol{X}_\gamma, \gamma)\|^2$, giving Eq.~\eqref{eq:est-denoiser-app}.

\medskip
\noindent\textbf{Score gap, Eq.~\eqref{eq:est-score-app}.}\quad
By Tweedie's formula, the score is a linear function of the optimal denoiser:
\begin{equation}
    \nabla\log p(x_\gamma) = \sqrt{\gamma}\,\hat{\boldsymbol{x}}^\star(x_\gamma, \gamma) - x_\gamma,
    \qquad
    \nabla\log p(x_\gamma \mid y) = \sqrt{\gamma}\,\hat{\boldsymbol{x}}^\star(x_\gamma, \gamma, y) - x_\gamma.
    \label{eq:tweedie}
\end{equation}
Subtracting cancels the $-x_\gamma$ term, so
\begin{equation*}
    \nabla\log p(x_\gamma \mid y) - \nabla\log p(x_\gamma) = \sqrt{\gamma}\bigl[\hat{\boldsymbol{x}}^\star(x_\gamma, \gamma, y) - \hat{\boldsymbol{x}}^\star(x_\gamma, \gamma)\bigr].
\end{equation*}
Squaring, taking expectation, and dividing by $2\gamma$ recovers Eq.~\eqref{eq:est-score-app} from Eq.~\eqref{eq:est-denoiser-app}.
The score-gap form has been widely adopted in mutual-information estimation work based on diffusion or score-matching models~\citep{minde, cls_spec}.

\medskip
\noindent\textbf{Fisher info gap, Eq.~\eqref{eq:est-fisher-app}.}\quad
By the law of total expectation, $\hat{\boldsymbol{x}}^\star(\boldsymbol{X}_\gamma, \gamma) = \E_{Y|\boldsymbol{X}_\gamma}[\hat{\boldsymbol{x}}^\star(\boldsymbol{X}_\gamma, \gamma, Y) \mid \boldsymbol{X}_\gamma]$ (the unconditional denoiser is the conditional expectation of the conditional one over $Y \mid \boldsymbol{X}_\gamma$); via Eq.~\eqref{eq:tweedie}, the same tower identity holds for the score:
\begin{equation*}
    \E_Y\!\left[\nabla\log p(\boldsymbol{X}_\gamma \mid Y) \mid \boldsymbol{X}_\gamma\right] = \nabla\log p(\boldsymbol{X}_\gamma).
\end{equation*}
Expanding the squared norm in Eq.~\eqref{eq:est-score-app}, the cross term reduces to $-2\E\|\nabla\log p(\boldsymbol{X}_\gamma)\|^2 = -2J(X_\gamma)$ by this identity, leaving
\begin{equation*}
    \E\bigl\|\nabla\log p(\boldsymbol{X}_\gamma \mid Y) - \nabla\log p(\boldsymbol{X}_\gamma)\bigr\|^2
    = J(X_\gamma \mid Y) - 2\, J(X_\gamma) + J(X_\gamma)
    = J(X_\gamma \mid Y) - J(X_\gamma).
\end{equation*}
Dividing by $2\gamma$ gives Eq.~\eqref{eq:est-fisher-app}.

\medskip
\noindent\textbf{Likelihood score, Eq.~\eqref{eq:est-lscore-app}.}\quad
This identity resonates with Classifier-Free Guidance \citep{cfg}: from $p(x_\gamma \mid y) = p(y \mid x_\gamma)\, p(x_\gamma) / p(y)$, take logarithm and gradient with respect to $x_\gamma$,
\begin{equation*}
    \nabla\log p(x_\gamma \mid y) - \nabla\log p(x_\gamma) = \nabla\log p(y \mid x_\gamma).
\end{equation*}
Substituting into Eq.~\eqref{eq:est-score-app} yields Eq.~\eqref{eq:est-lscore-app}.
\end{proof}

\paragraph{Comparison of estimators.}
\textcolor{red}{The five equivalent forms in Proposition~\ref{prop:equivalence-app} differ in numerical stability, the number of models required, and the SNR regime in which they are most reliable; Table~\ref{tab:estimators} summarizes the practical trade-offs.}

\begin{table}[h]
\centering
\small
\begin{tabular}{@{}lp{5cm}p{5cm}@{}}
\toprule
Estimator & Advantage & Disadvantage \\
\midrule
Single-feature MMSE gap & Directly reuses denoiser training; clear information-theoretic meaning & Requires two independent models; large-number subtraction amplifies variance \\
Denoiser gap & Equivalent to the single-feature MMSE gap but computed as a squared $\ell^2$ distance between two denoisers; numerically more stable & Still requires two denoiser models \\
Score gap & Natural interface with score-based models; equivalent to denoiser gap by Tweedie & Score training is less stable; the $1/\gamma$ factor amplifies errors at low SNR \\
Likelihood score & Only one model needed (a $y$-conditional likelihood-score predictor) & Requires a reverse-direction model; $1/\gamma$ singularity as $\gamma \to 0$ \\
\bottomrule
\end{tabular}
\caption{Comparison of estimators for the feature information density.}
\label{tab:estimators}
\end{table}

\paragraph{Asymmetry of mutual information estimation.}
A consequence of the equivalence family is that information density estimation can be done bidirectionally, while the difficulty can differ. Since $Y$ is a feature less complicated than $\boldsymbol{X}$, people might assume that learning a noisy classifier $p(Y \mid \boldsymbol{X}_{\gamma})$ is easier than learning generators $p(\boldsymbol{X}), p(\boldsymbol{X} \mid Y)$. This may hold for categorical $Y$ with a small number of classes, but for general continuous features, the likelihood can be difficult to estimate.

\paragraph{Continuous features and divergence.}
\label{app:alpha-robustness}
For discrete features, $I(Y; \boldsymbol{X})$ has an upper bound $H(Y)$ which is always finite. For general continuous features, however, $I(Y; \boldsymbol{X})$ can be infinite---in particular, when $Y = f(\boldsymbol{X})$ is a deterministic continuous function of $\boldsymbol{X}$. Adding observation noise $\alpha > 0$ to the feature---so that $Y_\alpha = f(\boldsymbol{X}) + \alpha\,\boldsymbol{\epsilon}$ with $\boldsymbol{\epsilon}\sim\mathcal{N}(0,I)$---restores finiteness. This regularization is not merely a numerical trick: it also reflects the realistic setting where features are measured with finite precision.

Two estimators in our framework respond very differently to this regularization. The \emph{empirical} estimator we use throughout the paper---linear interpolation followed by Gaussian smoothing on a finite log-$\gamma$ grid---never extrapolates, so it reports the chained gap on the support actually sampled. The \emph{analytic} estimator (Appendix~\ref{app:fitting}) instead fits a double-pole Lorentzian whose log-$\gamma$ image decays as $O(1/\gamma)$ in the right tail. For a deterministic continuous feature the true tail does \emph{not} decay (it is what produces the divergent $I(Y;\boldsymbol{X})$), so the Lorentzian's $O(1/\gamma)$ image is too thin to match it; on $\alpha=0$ data the fit therefore becomes degenerate, while on $\alpha=0.5$ data the same fit is well-conditioned. Figures~\ref{fig:cifar-alpha-compare} and~\ref{fig:alpha-lorentzian-failure} make both effects precise on CIFAR-10. \textcolor{red}{On the finite sampled SNR interval, the empirical summaries are close for many bands: the total chained mutual information across 12 bands agrees to within $0.32\%$, 9 of 12 bands agree to within $2\%$ in integrated MI, and the log-$\gamma$ density peak positions agree to within $0.005$ on mid and high bands. This does not make the two settings interchangeable: for the high band $b_{10}$, the $\alpha=0.5$ Lorentzian fit is well-conditioned ($\mu_{lg} = +1.49$), whereas the $\alpha=0$ fit collapses to the lower bound $b = 10^{-12}$ with apparent peak at $\mu_{lg} = +12.0$, far outside the data support, which ends at $\mu_{lg} = +1.46$.}

\begin{figure}[t]
  \centering
  \includegraphics[width=\textwidth]{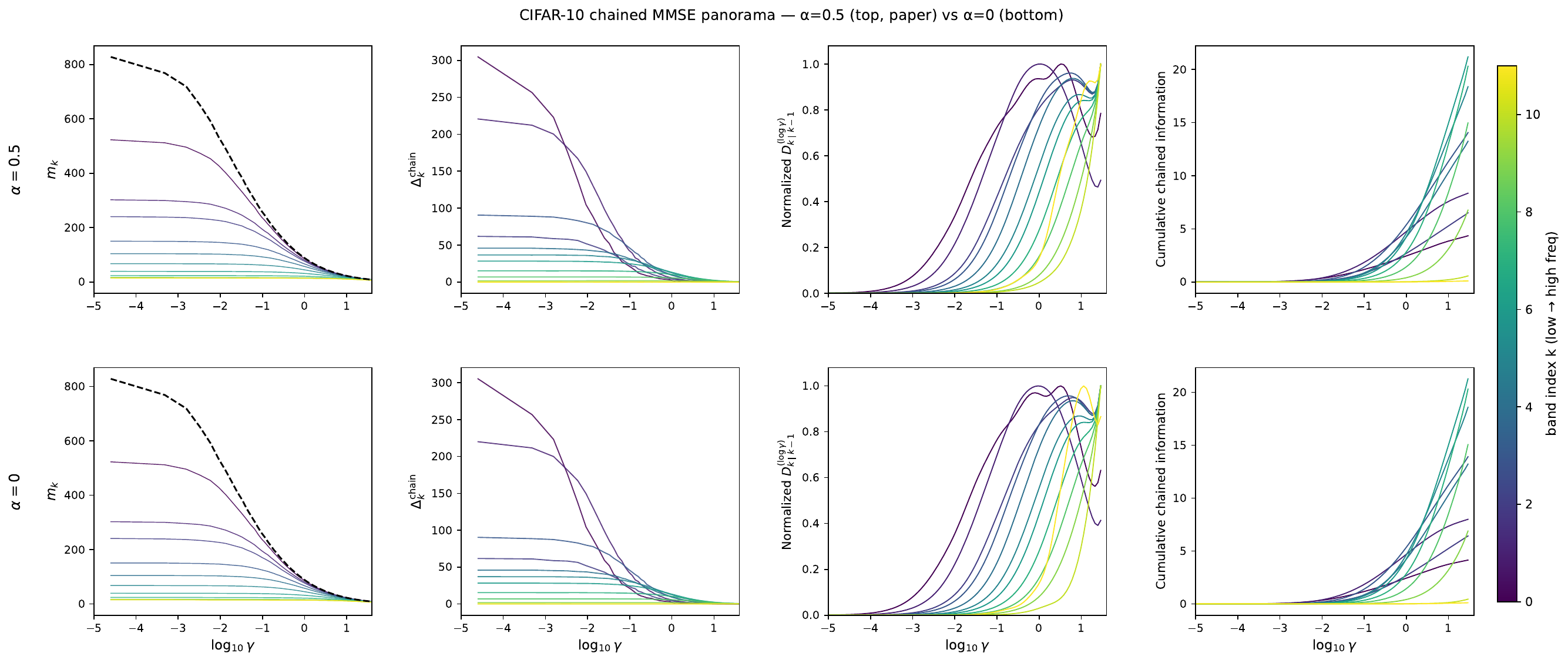}
  \caption{%
    \textcolor{red}{\textbf{Empirical finite-SNR summaries are close for many bands, but the analytic tail is not invariant to the feature-noise level $\alpha$.}}
    CIFAR-10 multiband chained MMSE diagnostics (12 frequency bands) at $\alpha = 0.5$ (top row) and $\alpha = 0$ (bottom row), with all other settings held fixed. From left to right: per-band MMSE $m_k$, chained gap $\Delta^{\mathrm{chain}}_k = m_{k-1} - m_k$, normalized log-$\gamma$ information density $D^{(\log\gamma)}_{k\mid k-1}$, and cumulative chained information. \textcolor{red}{The rows are close on most finite-SNR summaries, while the analytic high-SNR tail fit is unstable at $\alpha=0$ (Table~\ref{tab:alpha-fit-params}).}
  }
  \label{fig:cifar-alpha-compare}
\end{figure}

\begin{figure}[t]
  \centering
  \includegraphics[width=\textwidth]{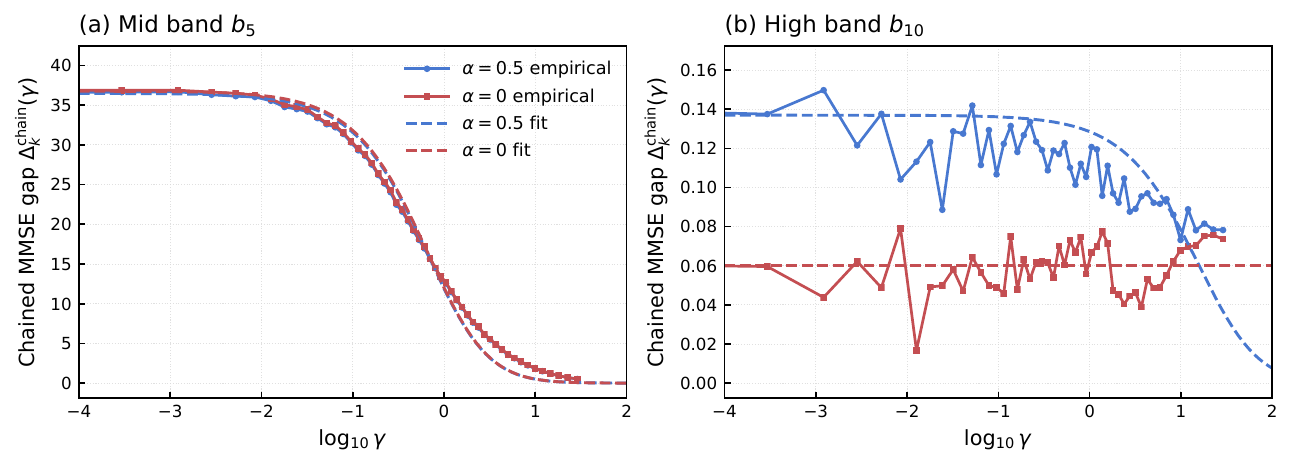}
  \caption{%
    \textbf{Why the analytic Lorentzian estimator requires $\alpha > 0$.}
    Two representative CIFAR-10 bands at $\alpha = 0.5$ (blue) and $\alpha = 0$ (red): the mid band $b_5$, where the chained MMSE gap decays cleanly with $\gamma$, and the high band $b_{10}$, where the gap is small and noisy. Solid lines/markers show the empirical chained MMSE gap $\Delta_k^{\mathrm{chain}}(\gamma)$ and dashed lines show the corresponding double-pole fits $L(\gamma)=a/(1+b\gamma)^2$. On $b_5$, both fits track the empirical curve and the fitted scales agree to within $0.001$ in $\mu_{lg}=\log_{10}(1/b)$ (Table~\ref{tab:alpha-fit-params}). On $b_{10}$, the $\alpha=0.5$ fit cleanly tracks the empirical decay, but the $\alpha=0$ fit collapses to a near-constant baseline---$b$ pinned at the optimizer lower bound $10^{-12}$, apparent peak at $\mu_{lg}=+12$, far outside the data support which ends at $+1.46$. This degeneracy is the analytic-estimator manifestation of the divergent $I(Y;\boldsymbol{X})$ for deterministic continuous features.
  }
  \label{fig:alpha-lorentzian-failure}
\end{figure}

\begin{table}[t]
  \centering
  \small
  \begin{tabular}{llccc}
    \toprule
    Band & Setting & $a$ & $b$ & $\mu_{lg}=\log_{10}(1/b)$ \\
    \midrule
    $b_{5}$  & $\alpha=0.5$ & $3.65\!\times\!10^{1}$  & $7.51\!\times\!10^{-1}$ & $+0.124$ \\
    $b_{5}$  & $\alpha=0$   & $3.68\!\times\!10^{1}$  & $7.50\!\times\!10^{-1}$ & $+0.125$ \\
    $b_{10}$ & $\alpha=0.5$ & $1.37\!\times\!10^{-1}$ & $3.23\!\times\!10^{-2}$ & $+1.491$ \\
    $b_{10}$ & $\alpha=0$   & $6.01\!\times\!10^{-2}$ & $\mathbf{1.00\!\times\!10^{-12}}$ & $\mathbf{+12.00}$ \\
    \bottomrule
  \end{tabular}
  \caption{Double-pole Lorentzian fit parameters $L(\gamma)=a/(1+b\gamma)^2$ for the two bands and two $\alpha$ settings shown in Figure~\ref{fig:alpha-lorentzian-failure}. Bold cells flag the $\alpha=0$ $b_{10}$ degeneracy where $b$ is pinned at the optimizer lower bound and the apparent peak location $\mu_{lg}$ is far outside the empirical $\log_{10}\gamma$ support (which ends at $+1.46$).}
  \label{tab:alpha-fit-params}
\end{table}


\section{Analytic Fitting of the Single-Feature MMSE Gap}
\label{app:fitting}

A well-trained diffusion model can be close to the optimal MMSE at moderate SNR, but the high-SNR regime is numerically delicate.
As $\gamma\to\infty$, the true MMSE and its training signal both become small; finite-capacity denoisers can therefore underfit this clean-data limit.
The resulting empirical MMSE is systematically too large, which contaminates the tail of the single-feature MMSE gap and can make a direct numerical integral unstable.
We address this issue by fitting the empirical single-feature MMSE-gap curve in a reliable region and extrapolating the high-SNR tail using the fitted curve.

\paragraph{Double-pole Lorentzian fitting.}
We exploit the fact that the theoretical single-feature MMSE-gap curve is smooth and decays to zero, admitting a good approximation by a sum of rational functions.
Specifically, we fit the empirical gap in the reliable region $\gamma < \gamma_{\max}$ (corresponding to $t < t_{\max}$) with a \textbf{Sum-of-Lorentzians} model:
\begin{equation}
  \hat{\D}(\gamma) = \sum_{k=1}^{K} \frac{a_k}{(1+b_k\,\gamma)^2},
  \qquad a_k, b_k > 0.
  \label{eq:lorentzians}
\end{equation}
Each component has an analytic integral $\int_0^\infty a_k/(1+b_k\gamma)^2\,d\gamma = a_k/b_k$, giving the total mutual information estimate:
\begin{equation}
  \hat{I}(Y;X) = \frac{1}{2}\sum_{k=1}^{K}\frac{a_k}{b_k}.
  \label{eq:mi-fit}
\end{equation}

\paragraph{Fitting boundary.}
We choose the fitting boundary from the empirical gap itself.
For each single-feature MMSE-gap curve, we first find the finite-grid index at which the gap reaches its minimum after the main decay, then move two grid points earlier and fit only the prefix before that cutoff.
This ``argmin minus two'' rule excludes the high-SNR tail where empirical denoiser underfitting and numerical artifacts are most visible.
The cutoff is shown as a vertical dashed line in the fitting diagnostics.
The fit is performed only on the reliable prefix and then extrapolated analytically to $t\to 1$.

\paragraph{Why fit the difference, not the individual curves.}
An apparently natural alternative is to fit the unconditional and conditional MMSE curves \emph{separately} and then take their difference.
This approach fails catastrophically: both curves are large, smooth, and nearly parallel, so their individually accurate fits produce a difference dominated by fitting residuals.
In our experiments, this ``fit-then-difference'' approach yielded mutual information estimates on the order of $-10^5$ nats---clearly meaningless.
Fitting the \emph{difference curve} $\hat{\D}(\gamma)$ directly avoids this issue entirely.

\paragraph{Model order selection.}
For the multiband MNIST experiments, a single double-pole component ($K=1$) already suffices.
The per-band fits achieve $R^2 \ge 0.98$, with most bands above $0.999$ in the reliable region.
This is why the fitted curves in Figure~\ref{fig:fitting-gap-grid} are smooth and stable despite using only two free parameters per band.

\begin{figure}[t]
  \centering
  \includegraphics[width=\textwidth]{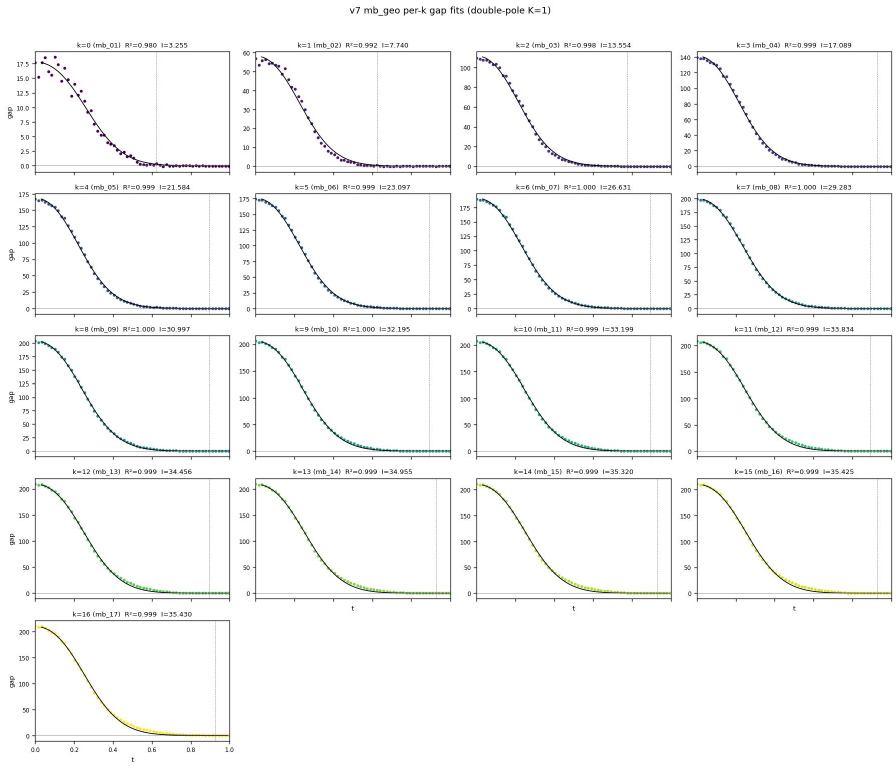}
  \caption{%
    \textbf{Double-pole Lorentzian fits to empirical single-feature MMSE gaps.}
    Dots show empirical per-band single-feature MMSE gaps and black curves show $K=1$ double-pole fits.
    The vertical dashed line marks the adaptive reliable cutoff used for fitting.
    Despite the simplicity of the parametric form, the fits track the empirical gaps surprisingly well across bands, including the decay region that dominates the mutual-information integral.
  }
  \label{fig:fitting-gap-grid}
\end{figure}


\section{Spectral Hierarchy Experiment Details}
\label{app:spectral-hierarchy-details}

This appendix records the implementation details for the spectral hierarchy experiment in Section~\ref{sec:exp:spectral-hierarchy}.

\paragraph{MNIST chained-control experiment.}
We decompose each MNIST image into 17 geometric frequency bands and use feature noise $\alpha=0.5$.
The main-paper Figure~\ref{fig:chained-order-comparison} keeps only the log-$\gamma$ information-density panels.
The appendix figure below shows the full four-panel diagnostic.
In the single-band setting, each model is conditioned on one band $Y_k$ independently, and the gap is $m_{\mathrm{uncond}}(\gamma)-m_k(\gamma)$.
In the forward chained setting, the conditions are nested from low to high frequency, $\mathcal{Y}_{\le k}=(Y_0,\ldots,Y_k)$, and the plotted increment is the chained MMSE gap $m_{k-1}(\gamma)-m_k(\gamma)$.
The checkpoints used in the figure are trained for 100k steps and re-evaluated on the full 10k-image test split with four independent noise draws per image.
We evaluate 61 diffusion times up to $t=0.995$.
The final saturated band is omitted from the panorama for visual clarity.

\begin{figure}[t]
  \centering
  \includegraphics[width=\textwidth]{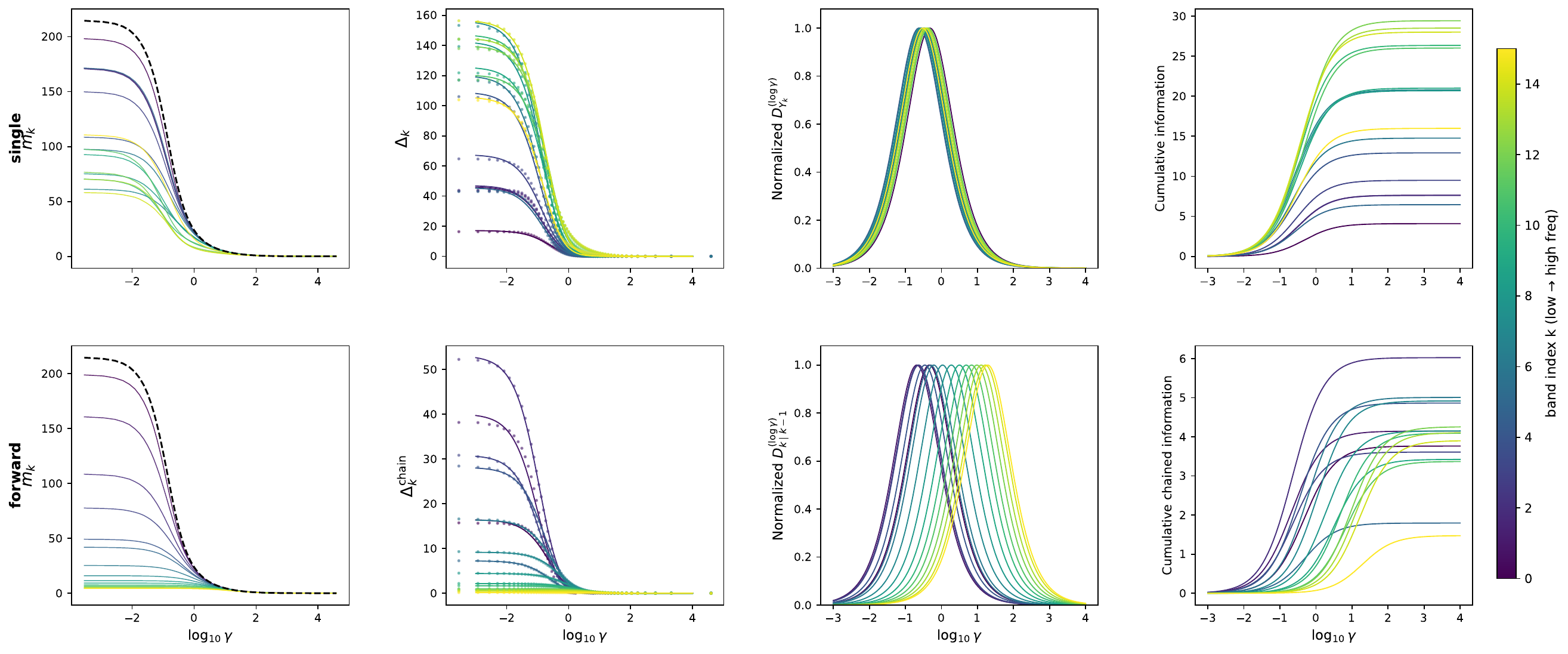}
  \caption{%
    \textbf{Detailed comparison of single-band and chained frequency decompositions on MNIST.}
    The full diagnostic includes MMSE, MMSE gap or chained MMSE gap, log-$\gamma$ information density, and cumulative information.
    The main text retains only the density row; the complete four-panel view is collected here for reference.
  }
  \label{fig:chained-order-comparison-detail}
\end{figure}

\paragraph{CIFAR-10 frequency experiment.}
We decompose each CIFAR-10 image into twelve DCT frequency bands $Y^{(0)},\ldots,Y^{(11)}$, ordered from the lowest band ($b_{00}$, DC and near-DC) to the highest band ($b_{11}$, finest detail).
The same three scheme types are used as in MNIST:
single-band models are conditioned on one band at a time, forward-chained models use the cumulative bundle $\mathcal{Y}_{\le k}=(Y^{(0)},\ldots,Y^{(k)})$, and reverse-chained models use the same bands in the opposite order.
The full diagnostic study included from-scratch small-convolutional, UNet, and pretrained-finetuning training lines.
Figure~\ref{fig:cifar10-freq} reports the final fine-tuning line because the earlier from-scratch models show the same qualitative monotone trend but are visibly limited by denoiser underfitting.
For the plotted CIFAR curve, we fine-tune each cumulative band-conditional model from the pretrained m04fm unconditional checkpoint at epoch 1999, using feature noise $\alpha=0.5$.
The chained MMSE gap is $\Delta^{\mathrm{chain}}_k(\gamma)=m_{k-1}(\gamma)-m_k(\gamma)$, with $m_{-1}$ given by the pretrained unconditional model.
The same log-$\gamma$ parametrization and incremental-gap definition are used as in the MNIST experiment.
We estimate the MMSE curves using 2048 test-split samples and 61 diffusion times.
The final two saturated high-frequency shells are merged as $b_{10}$--$b_{11}$ for readability in the plotted panorama.


\section{Representation Experiment Details}
\label{app:rep-exp-details}

This appendix records the paired experimental protocol used in Section~\ref{sec:exp:representation-analysis}.
The goal is to compare representations without letting data selection, condition construction, or inference plumbing become hidden confounders.

\paragraph{Data split and paired artifacts.}
All representation experiments use the same ImageNet-256 subset, denoted SR95.
SR95 contains the ImageNet classes for which our automatic mask pipeline succeeds on at least 95\% of images, yielding 485 classes.
We use a deterministic image-id split: 90\% of images for training and 10\% for validation/FID reference.
The training split contains 551,837 images, and the FID reference contains approximately 61k held-out images.
For the hierarchy experiments, each representation is evaluated on the same paired image list and the same train/validation whitelists.
Latent shards, mask shards, and canny shards are ordered by the same filename list, so a minibatch item carries the same semantic condition across pixel, SDVAE, VAVAE, and RAE runs.

\paragraph{Representations.}
The learnability comparison keeps the sequence length fixed at $T=256$ tokens.
Pixel space uses $256\times256\times3$ images with patch size 16.
SDVAE uses the standard $f{=}8,d{=}4$ latent space ($32\times32\times4$) with patch size 2.
VAVAE uses the $f{=}16,d{=}32$ latent space ($16\times16\times32$) with patch size 1.
RAE uses a DINOv2-B representation on a $16\times16$ grid with 768 channels and patch size 1.
The autoencoders/encoders are frozen; only the generative model is trained.

\paragraph{Mask construction.}
Object masks are produced once and then reused as fixed conditioning artifacts.
We query SAM~3.1~\citep{sam3} with the ImageNet class text prompt for each image.
The class-level SR95 filter keeps classes whose prompt succeeds on at least 95\% of images, with SAM confidence threshold 0.3.
Multiple available SAM passes are merged offline into a single binary foreground mask artifact per image; missing masks are represented by an all-zero sentinel and a validity flag.
For model input, each mask is resized by the same short-edge-to-256, center-crop geometry used for the corresponding image preprocessing, with nearest-neighbor interpolation.
The sharded training artifacts store masks as $256\times256$ uint8 tensors aligned with the image or latent shard order.

\paragraph{Canny construction.}
Canny conditions are derived from the same fixed masks, not recomputed during training.
For each raw RGB image, pixels outside the binary foreground mask are zeroed, the masked image is converted to grayscale, and OpenCV Canny is applied at thresholds $(100,200)$.
The resulting edge map is then center-cropped to $256\times256$ and stored as a uint8 shard aligned with the same filename order as the masks and latents.
Thus the canny condition is a local boundary cue inside the object support, rather than a full-image edge detector that can leak unrelated background structure.

\paragraph{Unconditional learnability search.}
For the unconditional FID trajectories discussed in Section~\ref{sec:exp:representation-analysis}, we first run a search stage with LDiT-B models.
The search uses global batch size 256, AdamW with $\beta_2=0.95$, bfloat16 training, Euler sampling with 50 function evaluations, no classifier-free guidance, and a force-null unconditional sampling path.
We sweep the logit-normal time shift used during training, $\alpha_t$, the shift used during sampling, $\alpha_s$, and the learning rate.
The coarse shift grid is
\[
\{0.34,\,0.56,\,1.0,\,1.78,\,2.95,\,4.63,\,6.93,\,10.0\}.
\]
Operationally, we first train one 20-epoch LDiT-B cell per $\alpha_t$ using $\alpha_s=\alpha_t$, and then evaluate the selected $\alpha_t$ across the remaining $\alpha_s$ values.
Learning-rate variants are run for the promising cells.
FID is computed with 10k generated samples against the SR95 held-out reference.
After this search stage and follow-up stability checks, the selected XL cells are trained for 200 epochs with checkpoints every 20 epochs.

\begin{table}[h]
\centering
\small
\begin{tabular}{@{}llllll@{}}
\toprule
Representation & Model & Prediction & Learning rate & $\alpha_t$ & $\alpha_s$ \\
\midrule
Pixel & LDiT-XL/16 & $x$ prediction & $5{\times}10^{-5}$ & 1.0 & 1.0 \\
SDVAE & LDiT-XL/2 & velocity & $1{\times}10^{-4}$ & 1.0 & 1.0 \\
VAVAE & LDiT-XL/1 & velocity & $1{\times}10^{-4}$ & 1.78 & 1.78 \\
RAE & LDiT-XL/1 & $x$ prediction & $2{\times}10^{-4}$ & 0.10 & 0.10 \\
\bottomrule
\end{tabular}
\caption{Selected XL cells used for the representation learnability trajectories. RAE required an additional boundary-extension check beyond the original coarse shift grid, which selected $\alpha_t=\alpha_s=0.10$.}
\label{tab:learnability-selected-cells}
\end{table}

\paragraph{Corrected force-null inference.}
All learnability runs are unconditional.
During training, class labels are dropped with probability 1.0.
During sampling and FID evaluation, the model is explicitly called with the null class label rather than with accidental ImageNet class labels.
This correction matters: before fixing the label plumbing, unconditional checkpoints could show an artificial late-epoch FID degradation caused by evaluating an unconditional model through a conditional label path.
Under the corrected force-null protocol, all four selected XL trajectories improve monotonically or near-monotonically over the evaluated checkpoint range.

\paragraph{Chained hierarchy protocol.}
The hierarchy experiment trains a separate checkpoint for each link in the chain
\[
\text{uncond}\rightarrow\text{class}\rightarrow\text{class+mask}\rightarrow\text{class+mask+canny},
\]
applied identically to all four representations (pixel, SDVAE, VAVAE, RAE).
The class stage adapts the base generator to SR95 with labels available.
The mask stage is initialized from the accepted class-stage checkpoint and adds the fixed foreground mask condition.
The canny stage adds the fixed masked-canny condition on top of class and mask.
Validation MSE on the paired validation split selects the checkpoint used for MMSE evaluation.
For latent representations, mask and canny tensors are loaded from sharded sidecars aligned with the latent shard; for pixel and online-RAE runs, raw images, masks, and canny maps are loaded jointly and undergo synchronized horizontal flipping during training.
For the pixel cell, the accepted checkpoints reach validation MMSE $0.0540$ (class), $0.0518$ (class+mask), and $0.0516$ (class+mask+canny); because the class-stage checkpoint is already close to the pixel-space reconstruction floor, the absolute validation-MMSE gains are small but the chained information-density curves still localize the remaining conditional information along log-SNR.

\paragraph{Mode dispatch.}
All MMSE evaluations use four explicit modes on the same paired validation examples:
\[
\begin{array}{ll}
  \text{mode 1}: & \text{unconditional, with the null class and zero mask/canny tensors},\\
  \text{mode 2}: & \text{class, with the real class and zero mask/canny tensors},\\
  \text{mode 3}: & \text{class+mask, with the real class and foreground mask},\\
  \text{mode 4}: & \text{class+mask+canny, with the real class, foreground mask, and masked canny map}.
\end{array}
\]
The plotted increments are adjacent chained MMSE gaps,
$\Delta^{\mathrm{chain}}_{\mathrm{class}}=m_1-m_2$,
$\Delta^{\mathrm{chain}}_{\mathrm{mask}}=m_2-m_3$,
$\Delta^{\mathrm{chain}}_{\mathrm{canny}}=m_3-m_4$, each converted to the log-SNR density $D(u)=\tfrac{1}{2}\gamma\Delta^{\mathrm{chain}}(\gamma)$ with $u=\log\gamma$ per Eq.~\eqref{eq:chained-density-def}.
Each representation uses its native model interface to evaluate these modes; the trained checkpoints serve as a chained measurement device, not as ablations in which one checkpoint is queried with missing conditions outside its training distribution.

\begin{table}[h]
\centering
\small
\begin{tabular}{@{}lllll@{}}
\toprule
Representation & Generator & Data space & Batch size & Conditioning stages \\
\midrule
Pixel & JiT-L/16 & raw images & 1024 & class, class+mask, class+mask+canny \\
SDVAE & SiT-XL/2 & $32{\times}32{\times}4$ latents & 256 & class, class+mask, class+mask+canny \\
VAVAE & LDiT-XL/1 & $16{\times}16{\times}32$ latents & 1024 & class, class+mask, class+mask+canny \\
RAE & DiTDH-XL & DINOv2-B RAE tokens & 1024 & class, class+mask, class+mask+canny \\
\bottomrule
\end{tabular}
\caption{Backbone and data-space settings for the chained hierarchy measurements. Each representation is trained as a chain of three nested-condition stages---class, class+mask, class+mask+canny---initialized from the previous stage's checkpoint. The conditioning artifacts are shared across representations; only the native representation encoder/backbone and the required optimizer scale differ.}
\label{tab:hierarchy-representation-settings}
\end{table}

\paragraph{Conditioning adapter comparison.}
For the mask and canny phases, we also evaluated a FiLM-style adaptive conditioning adapter against the prior conditioning implementation used in the corresponding accepted runs.
Table~\ref{tab:film-adapter-comparison} reports the best validation loss reached by each variant and the epoch at which it occurred.
The comparison is used only as an implementation diagnostic for the hierarchy experiments: FiLM gives clear gains for VAVAE and RAE, small or negligible gains for pixel space, and does not improve SDVAE under the tested online setting.

\begin{table}[h]
\centering
\small
\resizebox{\textwidth}{!}{%
\begin{tabular}{@{}llccrrl@{}}
\toprule
Representation & Phase & Prior best val. loss & FiLM best val. loss & $\Delta$ & $\Delta$ (\%) & Outcome \\
\midrule
Pixel & mask  & 0.053708 (45 ep.) & 0.053690 (51 ep.) & $-0.000018$ & $-0.03$ & tied \\
Pixel & canny & 0.051556 (54 ep.) & 0.051283 (66 ep.) & $-0.000273$ & $-0.53$ & FiLM slightly better \\
SDVAE & mask  & 0.719400 (42 ep.) & 0.723878 (30 ep.) & $+0.004478$ & $+0.62$ & prior better \\
SDVAE & canny & 0.692970 (60 ep.) & 0.698930 (30 ep.) & $+0.005960$ & $+0.86$ & prior better \\
VAVAE & mask  & 0.274291 (30 ep.) & 0.270818 (51 ep.) & $-0.003473$ & $-1.27$ & FiLM better \\
VAVAE & canny & 0.268236 (15 ep.) & 0.265441 (60 ep.) & $-0.002795$ & $-1.04$ & FiLM better \\
RAE & mask  & 0.283094 (120 ep.) & 0.273781 (84 ep.) & $-0.009313$ & $-3.29$ & FiLM better \\
RAE & canny & 0.283082 (24 ep.) & 0.271056 (45 ep.) & $-0.012025$ & $-4.25$ & FiLM better \\
\bottomrule
\end{tabular}%
}
\caption{Validation-loss comparison between the prior conditioning implementation and a FiLM-style adaptive conditioning adapter for the mask and canny phases. Negative $\Delta$ means FiLM improves validation loss.}
\label{tab:film-adapter-comparison}
\end{table}

\paragraph{MMSE evaluation.}
Each representation uses its native model interface for denoising, but the exported quantities have the same semantics: unconditional, class, class+mask, and class+mask+canny MMSE curves.
The downstream computation is shared across representations: adjacent curves define chained MMSE gaps, and each chained MMSE gap is converted to a log-SNR information density using Eq.~\eqref{eq:chained-density-def}.
This common post-processing is what makes the hierarchy comparison in Figure~\ref{fig:rep-hierarchy} representation-level rather than script-level.
Figure~\ref{fig:rep-hierarchy-panorama} collects the full panorama diagnostics for all four representations; the main text keeps only the compact density summary.

\begin{figure}[t]
  \centering
  \begin{subfigure}[t]{\textwidth}
    \centering
    \includegraphics[width=\textwidth]{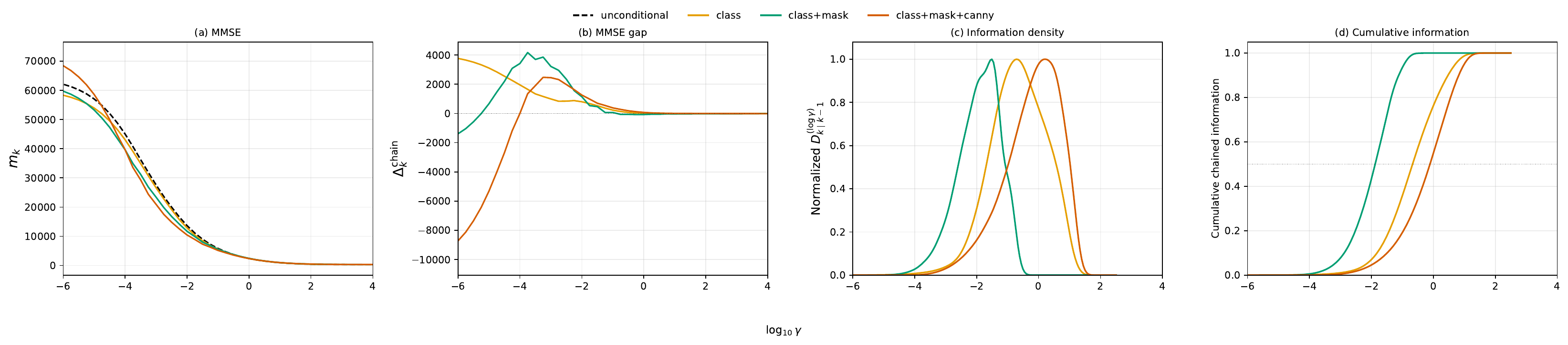}
    \caption{Pixel}
  \end{subfigure}
  \vspace{0.4em}
  \begin{subfigure}[t]{\textwidth}
    \centering
    \includegraphics[width=\textwidth]{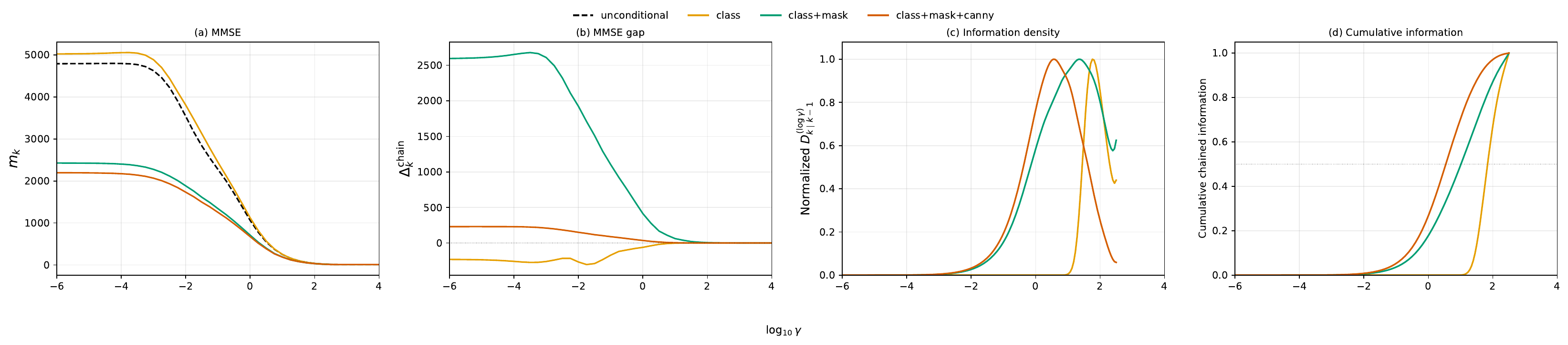}
    \caption{SDVAE}
  \end{subfigure}
  \vspace{0.4em}
  \begin{subfigure}[t]{\textwidth}
    \centering
    \includegraphics[width=\textwidth]{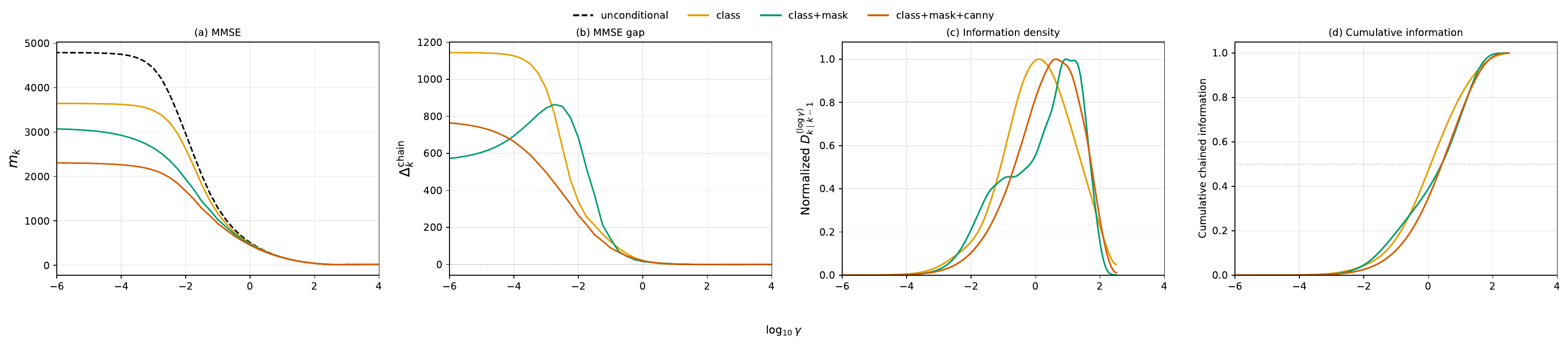}
    \caption{VAVAE}
  \end{subfigure}
  \vspace{0.4em}
  \begin{subfigure}[t]{\textwidth}
    \centering
    \includegraphics[width=\textwidth]{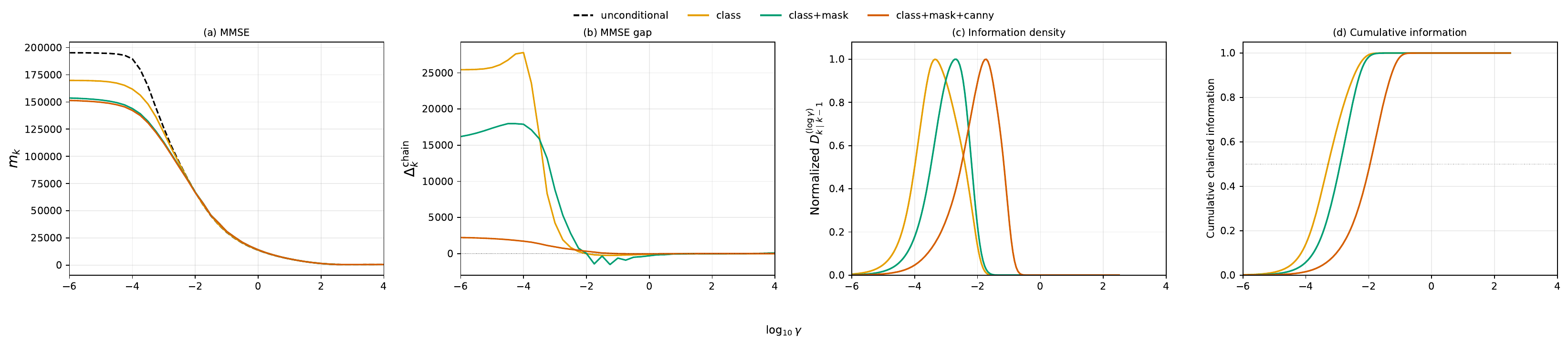}
    \caption{RAE}
  \end{subfigure}
  \caption{%
    \textbf{Full chained-hierarchy panoramas across representations.}
    Each row reports the same four diagnostics for one representation: MMSE curves, chained MMSE gaps, log-SNR information densities, and cumulative information for class, mask, and canny increments. The main text summarizes the density column in Figure~\ref{fig:rep-hierarchy}; this appendix figure keeps the full diagnostic view for all four representations.
  }
  \label{fig:rep-hierarchy-panorama}
\end{figure}

\paragraph{Fairness controls and unavoidable differences.}
The fairness-critical controls are shared data, shared condition artifacts, paired train/validation whitelists, aligned shard order, identical MMSE evaluation grids within each comparison, and the same chained-conditioning semantics.
The unavoidable differences are those required by the representation itself: native patch size, channel dimension, encoder family, prediction parameterization, and stable learning-rate range.
We therefore interpret Section~\ref{sec:exp:representation-analysis} as a representation-level comparison under matched experimental plumbing, not as a claim that a single architecture and optimizer setting is universally optimal across incompatible data spaces.

\paragraph{\textcolor{red}{Checkpoint, model-size, and order robustness.}}
\label{app:hierarchy-robustness}
\textcolor{red}{The completed alternative-order evaluation uses the RAE DiTDH-XL configuration (\texttt{DiTDH-XL\_DINOv2-B\_offline.yaml}), with 768 latent channels and the matched mask-first evaluation protocol recorded in the accompanying did\_q01 artifact. The convergence check compares approximately 132M-parameter RAE-S checkpoints at 20, 40, 60, and 80 epochs with the approximately 842M-parameter RAE/DiTDH-XL paper model. The reported peak locations are listed in Table~\ref{tab:camera-ready-stability}; the alternative order is evaluated on the XL configuration, while the checkpoint sweep uses RAE-S. The broader experiment log also compares pretrained-checkpoint initialization with from-scratch training, and shows that the resulting measured gaps need not support the same qualitative conclusion.}

\clearpage

\section*{NeurIPS Paper Checklist}

\begin{enumerate}

\item {\bf Claims}
    \item[] Question: Do the main claims made in the abstract and introduction accurately reflect the paper's contributions and scope?
    \item[] Answer: \answerYes{}
    \item[] Justification: Our paper’s contributions and scope are presented in the abstract and introduction accurately.

\item {\bf Limitations}
    \item[] Question: Does the paper discuss the limitations of the work performed by the authors?
    \item[] Answer: \answerYes{}
    \item[] Justification: A dedicated \emph{Limitations} paragraph in Section~\ref{sec:discussion} discusses the limitations of our work.

\item {\bf Theory assumptions and proofs}
    \item[] Question: For each theoretical result, does the paper provide the full set of assumptions and a complete (and correct) proof?
    \item[] Answer: \answerYes{}
    \item[] Justification: All theoretical objects --- the I-MMSE identity, the feature information density (Eq.~\eqref{eq:feature-density-gamma}), the log-SNR re-parameterization (Eq.~\eqref{eq:feature-density-logsnr}), the chained density (Eq.~\eqref{eq:chained-density-def}), and the equivalent-estimator family --- state their assumptions inline in the \emph{Feature Information Density} and \emph{Chained Information Decomposition} sections (Gaussian channel with $\boldsymbol{X}_\gamma=\sqrt{\gamma}\boldsymbol{X}+\boldsymbol{N}$, $\boldsymbol{N}\!\perp\!(\boldsymbol{X},Y)$, mild regularity for the de Bruijn / chain-rule steps). Full derivations and the equivalence proofs are given in Appendix~\ref{app:estimators}; the Sum-of-Lorentzians fitting and tail-extrapolation argument is in Appendix~\ref{app:fitting}. All definitions and equations are numbered and cross-referenced.

    \item {\bf Experimental result reproducibility}
    \item[] Question: Does the paper fully disclose all the information needed to reproduce the main experimental results of the paper to the extent that it affects the main claims and/or conclusions of the paper (regardless of whether the code and data are provided or not)?
    \item[] Answer: \answerYes{}
    \item[] Justification: Section~\ref{sec:exp:representation-analysis} specifies the data construction (SAM~3.1-validated SR95 ImageNet-256 subset), conditioning chain, parameterization (Flow Matching loss with $\boldsymbol{x}$-prediction for pixel/RAE and velocity for SDVAE/VAVAE), and density-conversion equation. Appendix~\ref{app:rep-exp-details} additionally documents the deterministic image-id split, paired mask/Canny shard alignment, the per-representation token count and patch size, the preliminary learnability search grid, the corrected force-null inference protocol, the MMSE evaluation procedure, and the selected XL-cell hyperparameters in Tables~\ref{tab:learnability-selected-cells} and~\ref{tab:hierarchy-representation-settings}; Appendix~\ref{app:spectral-hierarchy-details} similarly documents the spectral-hierarchy experiment, and the pixel instance of the representation hierarchy is documented inside Appendix~\ref{app:rep-exp-details} alongside the latent representations.

\item {\bf Open access to data and code}
    \item[] Question: Does the paper provide open access to the data and code, with sufficient instructions to faithfully reproduce the main experimental results, as described in supplemental material?
    \item[] Answer: \answerYes{}
    \item[] Justification: \textcolor{red}{Code and reproduction instructions will be released at \url{https://github.com/AI4Science-WestlakeU/feature-information-dynamics}.}

\item {\bf Experimental setting/details}
    \item[] Question: Does the paper specify all the training and test details (e.g., data splits, hyperparameters, how they were chosen, type of optimizer) necessary to understand the results?
    \item[] Answer: \answerYes{}
    \item[] Justification: Optimizer (AdamW, $\beta_2{=}0.95$, bfloat16), learning rate, batch size, time-shift parameters $(\alpha_t,\alpha_s)$, prediction parameterization, training/validation split (90/10 SR95, 551{,}837 train images), token count, patch size, sampler (Euler, 50 NFE), and hyperparameter-selection protocol (LDiT-B grid search prior to XL training) are documented in Appendix~\ref{app:rep-exp-details} and summarized in Tables~\ref{tab:learnability-selected-cells}--\ref{tab:hierarchy-representation-settings}. Spectral-hierarchy details are in Appendix~\ref{app:spectral-hierarchy-details}; the pixel-instance hyperparameters are in Appendix~\ref{app:rep-exp-details} together with the latent representations.

\item {\bf Experiment statistical significance}
    \item[] Question: Does the paper report error bars suitably and correctly defined or other appropriate information about the statistical significance of the experiments?
    \item[] Answer: \answerNo{}
    \item[] Justification: Each cross-representation chain requires four XL-scale flow-matching checkpoints (unconditional, class, class$+$mask, class$+$mask$+$Canny) for each of four representations, so multi-seed runs are computationally infeasible within the project's compute budget; the reported densities are therefore single-seed trajectories. Per-image MMSE quantities are averaged over the full SR95 validation split with multiple noise draws per image (Appendix~\ref{app:rep-exp-details}), and the qualitative ordering of the chained density peaks (the central claim) is robust across checkpoint cadence rather than across seeds. We mark this as a limitation in Section~\ref{sec:discussion}.

\item {\bf Experiments compute resources}
    \item[] Question: For each experiment, does the paper provide sufficient information on the computer resources (type of compute workers, memory, time of execution) needed to reproduce the experiments?
    \item[] Answer: \answerYes{}
    \item[] Justification: The cross-representation comparison uses XL-scale flow-matching backbones (LDiT-XL/JiT-L/SiT-XL/DiTDH-XL), batch sizes 256--1024, bfloat16 precision, 200 epochs of unconditional training plus three additional conditioning phases per representation. Appendix~\ref{app:rep-exp-details} reports these settings together with the preliminary LDiT-B grid-search budget; per-experiment GPU type and total wall-clock are summarized in the same appendix.

\item {\bf Code of ethics}
    \item[] Question: Does the research conducted in the paper conform, in every respect, with the NeurIPS Code of Ethics \url{https://neurips.cc/public/EthicsGuidelines}?
    \item[] Answer: \answerYes{}
    \item[] Justification: The paper analyzes existing image-diffusion models on standard public benchmarks (ImageNet, MNIST, CIFAR-10) and does not collect new human-subject data, deploy new generative content, or release new pre-trained generative artifacts. Authors have reviewed the NeurIPS Code of Ethics and the work conforms with it.

\item {\bf Broader impacts}
    \item[] Answer: \answerNA{}
    \item[] Question: Does the paper discuss both potential positive societal impacts and negative societal impacts of the work performed?
    \item[] Justification: The paper is foundational analysis: it introduces a diagnostic framework for studying when features are generated along the diffusion trajectory and applies it to existing public diffusion models. It does not release new generative models, new training data, or improved generation quality, so there is no direct path from this work to a societal-impact application. We therefore mark broader impact as not applicable in the body of the paper.

\item {\bf Safeguards}
    \item[] Question: Does the paper describe safeguards that have been put in place for responsible release of data or models that have a high risk for misuse (e.g., pre-trained language models, image generators, or scraped datasets)?
    \item[] Answer: \answerNA{}
    \item[] Justification: The paper does not release new generative models, new pre-trained language/image artifacts, or scraped datasets that pose a misuse risk. 

\item {\bf Licenses for existing assets}
    \item[] Question: Are the creators or original owners of assets (e.g., code, data, models), used in the paper, properly credited and are the license and terms of use explicitly mentioned and properly respected?
    \item[] Answer: \answerYes{}
    \item[] Justification: All external assets used --- ImageNet (Russakovsky et al.), MNIST, CIFAR-10, SAM~3.1~\citep{sam3}, DINOv2, SDVAE~\citep{ldm}, VAVAE~\citep{vavae}, RAE~\citep{rae}, JiT~\citep{jit} --- are properly cited in the references and used within the terms of their public releases. The SR95 subset is constructed deterministically from public ImageNet train splits.

\item {\bf New assets}
    \item[] Question: Are new assets introduced in the paper well documented and is the documentation provided alongside the assets?
    \item[] Answer: \answerNA{}
    \item[] Justification: The paper does not release new generative models, datasets, or code as a primary asset. The SR95 mask and Canny shards are intermediate experiment-specific artifacts derived deterministically from public ImageNet and SAM~3.1 outputs; if released as supplementary material, they will be accompanied by the construction script described in Appendix~\ref{app:rep-exp-details}.

\item {\bf Crowdsourcing and research with human subjects}
    \item[] Question: For crowdsourcing experiments and research with human subjects, does the paper include the full text of instructions given to participants and screenshots, if applicable, as well as details about compensation (if any)?
    \item[] Answer: \answerNA{}
    \item[] Justification: The paper does not involve crowdsourcing or research with human subjects.

\item {\bf Institutional review board (IRB) approvals or equivalent for research with human subjects}
    \item[] Question: Does the paper describe potential risks incurred by study participants, whether such risks were disclosed to the subjects, and whether Institutional Review Board (IRB) approvals (or an equivalent approval/review based on the requirements of your country or institution) were obtained?
    \item[] Answer: \answerNA{}
    \item[] Justification: The paper does not involve research with human subjects, so IRB approval is not applicable.

\item {\bf Declaration of LLM usage}
    \item[] Question: Does the paper describe the usage of LLMs if it is an important, original, or non-standard component of the core methods in this research? Note that if the LLM is used only for writing, editing, or formatting purposes and does \emph{not} impact the core methodology, scientific rigor, or originality of the research, declaration is not required.
    \item[] Answer: \answerNA{}
    \item[] Justification: LLMs were used only for writing assistance, editing, and code/figure refactoring; they are not part of the core methodology, theoretical derivations, or experimental design of the diffusion-information-dynamics framework.

\end{enumerate}

\end{document}